\documentclass{llncs}
\usepackage{makeidx}  % allows for indexgeneration
\usepackage{graphicx} % for pdf, bitmapped graphics files
\usepackage{grffile} % to sovle the problem of unknow png format
\usepackage{epsfig} % for postscript graphics files
\usepackage{amsmath}
\usepackage{amssymb}
\usepackage{subcaption}
\usepackage{mathtools}
\usepackage{cite}
\usepackage{booktabs}
\usepackage{hyperref}
\hypersetup{
    colorlinks=true,
    linkcolor=blue,
    citecolor=red,
}

\newtheorem{formation}{Formation Guideline}
\def\acknowledgement{\par\addvspace{17pt}\small\rmfamily
\trivlist\if!\ackname!\item[]\else
\item[\hskip\labelsep
{\bfseries\ackname}]\fi}

\newenvironment{acknowledgements}{\begin{acknowledgement}}
{\end{acknowledgement}}

\title{Distributed Circumnavigation Control with Dynamic spacing for a Heterogeneous Multi-robot System}

\author{Weijia Yao\inst{1} \and Sha Luo\inst{1} \and Huimin Lu\inst{1} \and Junhao Xiao\inst{1}}
\institute{Department of Automation, National University of Defense Technology, Changsha 410073, P.~R.~China\\ \email{weijia.yao.nudt@gmail.com, luoshasha1992@gmail.com, lhmnew@nudt.edu.cn, junhao.xiao@hotmail.com} }

\begin{document}

\maketitle
\thispagestyle{empty}
\pagestyle{empty}

\begin{abstract}

Circumnavigation control is useful in real-world applications such as entrapping a hostile target. In this paper, we consider a heterogeneous multi-robot system where robots have different physical properties, such as maximum movement speeds. Instead of equal-spacing which is assumed in many existing studies, dynamic spacing according to robots' properties is proposed in this paper. For this purpose, two new concepts - \textit{utility} and \textit{formation guideline} - are presented. Then a distributed circumnavigation control algorithm based on utilities and formation guidelines is designed for any number of mobile robots from random 3D positions to circumnavigate a target. Theoretical analysis and experimental results are provided to prove the stability and effectiveness of the proposed control algorithm. 

\end{abstract}

% Please remove or comment out the following line if the footnote is not necessary
% \footnotetext{This work is supported by National Natural Science
% Foundation (NNSF) of China under Grant 00000000.}

\section{Introduction}

%More and more research focuses on multi-robot systems since there are many advantages over single-robot systems \cite{wang2016}. The ways of controlling a multi-robot system can be categorized into two classes: centralized control and distributed control. The prerequisite for centralized control is that there is at least one robot (the central robot) that is able to communicate with all the other robots \cite{ren2008}, which is not always possible in many real scenarios. In contrast, distributed control employs local information among robots to realize collective behaviour and achieve a global objective \cite{wang2016}. 

One of the most prominent research topics on distributed multi-robot system is the formation control problem. Significant efforts have been made on the \textit{circular formation control} and \textit{circumnavigation control} problems. In circular formation control problem, robots remain in their positions after the formation is generated, while in circumnavigation control problem, they still encircle around the target. In this sense, circular formation control could be regarded as a special case of circumnavigation control when the circumnavigation speed equals to zero. 
	
There are already many studies on circumnavigation control (or circular formation control) problems. Most of the existing studies only consider the case where robots are distributed evenly on the formation (i.e., equal spacing), such as \cite{Marshall2015, Pavone2007, Yu2016, Zheng2015a, Tang2014, zheng2013distributed}. In addition, the control algorithms proposed in these studies are only applicable on the 2D plane.  Nevertheless, \cite{franchi2016} proposes algorithms which are still effective in 3D space. The formation spacing, however, is fixed and equal. Although this is effective for a homogeneous multi-robot system, it may not be sufficient for a heterogeneous one where robots have different properties, such as maximum movement speeds. \cite{wang2013} and \cite{wang2014controlling} propose a distributed control law for a multi-robot system to form a circular formation with any desired spacing among robots. However, it assumes that the robots are placed initially on a prescribed circle and the control algorithm is not applicable in the 3D space. Another major disadvantage is that the desired spacing, which is a global quantity, should be specified for each robot beforehand. If the specified spacing does not sum up to $2 \pi$, or if robots are informed of inconsistent specified spacing, they will form an erroneous formation. Moreover, to the best of our knowledge, there are no studies concerning dynamic spacing for a heterogeneous multi-robot system. 
	
In this study, we suppose that mobile robots are heterogeneous in terms of their kinematics abilities, such as maximum locomotion speeds, etc. In a scenario where these mobile robots need to entrap a hostile target, their inter-robot spacing should be different for better performance; those robots with lower mobility are supposed to gather together with smaller spacing than those with higher mobility, so the probability for the target to flee away from the formation is lower. We also consider the deterioration of individual performance due to physical worn-out or damage. Therefore, their spacing should be varied in a dynamic way during the circumnavigation process. Based on this, the goal of this paper is to propose a new distributed circumnavigation control algorithm which is able to control a group of heterogeneous mobile robots from any initial positions to circumnavigate a target with dynamic spacing in the 3D space. 

The main contribution of this work is twofold. First, this paper proposes the concept of \textit{utility} and \textit{formation guideline}. Based on these two new concepts, we design a distributed circumnavigation control algorithm which enables robots to adjust their spacing dynamically according to the local variations of their utilities; a pre-specified desired spacing is not necessary (but it is necessary for studies such as \cite{franchi2016, wang2013, wang2014controlling}). The control algorithm is distributed and applicable for a heterogeneous multi-robot system of arbitrary size. Second, the distributed control algorithm does not require robots to be placed initially on a prescribed circular trajectory (but it is required in \cite{wang2013, wang2014controlling}). Their initial positions can be arbitrarily chosen in the 3D space rather than being restricted on a 2D plane (which is the case in \cite{wang2013, wang2014controlling}). In addition, the control algorithm can respond effectively to the situation where robots quit or join the circumnavigation process (but it is not studied in much literature such as \cite{Marshall2015, Tang2014, wang2013, wang2014controlling}).

%\begin{enumerate}
%\item This paper put forward the concept of \textit{utility} and \textit{formation guideline}. Based on these two new conepts, we design a distributed circumnavigation control algorithm which enables robots to adjust their spacing dynamically according to the local variation of the utilities. Note that the spacing among robots are not necessarily identical as most of the existing studies assume. The control algorithm is distribute and is applicable for a heterogeneous multi-robot system of arbitrary size.
%\item The distributed control algorithm does not require robots to be placed initially on a prescribed circular trajectory. Their initial positions can be arbitrarily chosen in the 3D space.
%\item This control algorithm can repond effectively to the situation where robots quit or join the circumnavigation process.
%\end{enumerate}
	
The remainder of this paper is organized as follows. Section \ref{sec2} introduces the circumnavigation control problem based on utilities and derives its corresponding mathematical formulation. Section \ref{sec3} proposes the circumnavigation control algorithm. In section \ref{sec4}, simulation and real-robot experiments are performed and results are analysed. Finally, section \ref{sec5} concludes the paper and summarizes the future work.
% % only one pic
%	\begin{figure}[!htb]
%	  \centering
%	  \includegraphics[scale=0.6]{figures/cylin_xy.png}
%	  \caption{The illustration of the multi-robot circumnavigation control problem}
%	  \label{fig1}
%	\end{figure}

% % fig 1 and fig 2
	\begin{figure}[htb]
		\begin{minipage}{0.3\columnwidth}
		 \centering
		 \includegraphics[height=\linewidth]{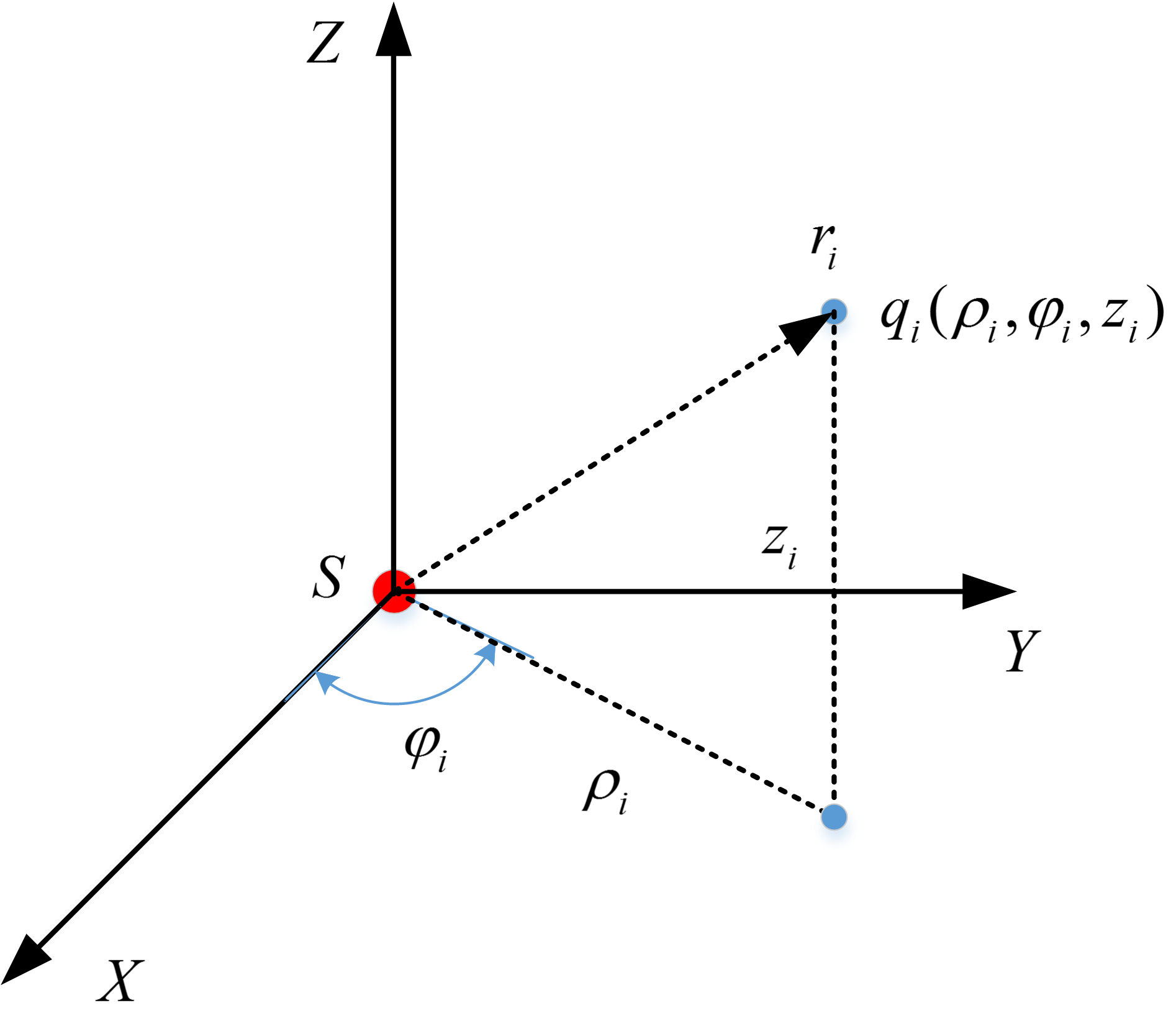}
		 \caption{The body reference frame with the target $S$ as the origin.}
		 \label{circum_pic1}
		\end{minipage}\hfill
		\begin{minipage}{0.3\columnwidth}
		 \centering
		 \includegraphics[height=\linewidth]{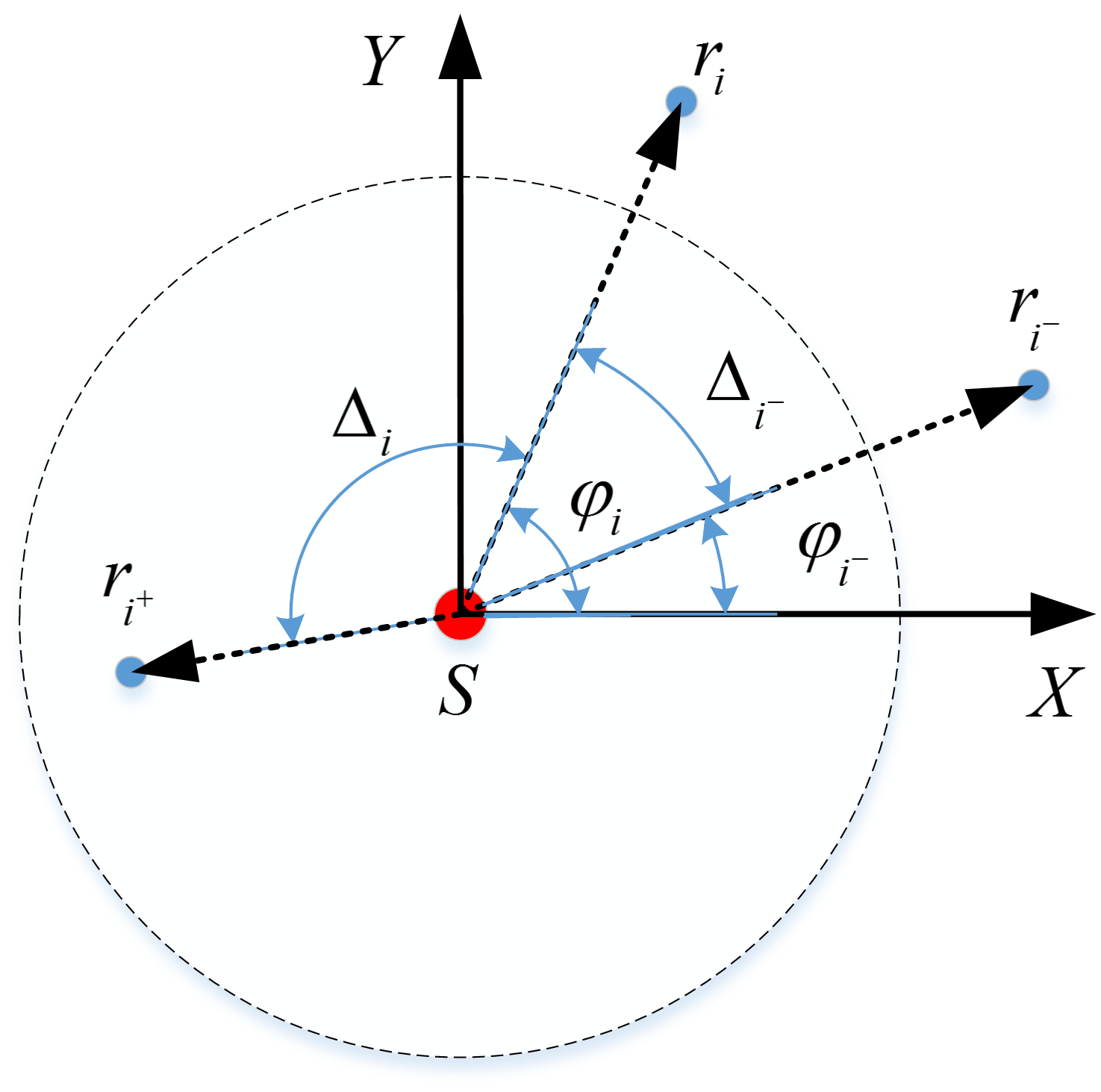}
		 \caption{Robots' projections and the target $S$ on the $X S Y$ plane.}
		 \label{circum_pic2}
		\end{minipage}\hfill
		\begin{minipage}{0.3\columnwidth}
		 \centering
		 \includegraphics[height=\linewidth]{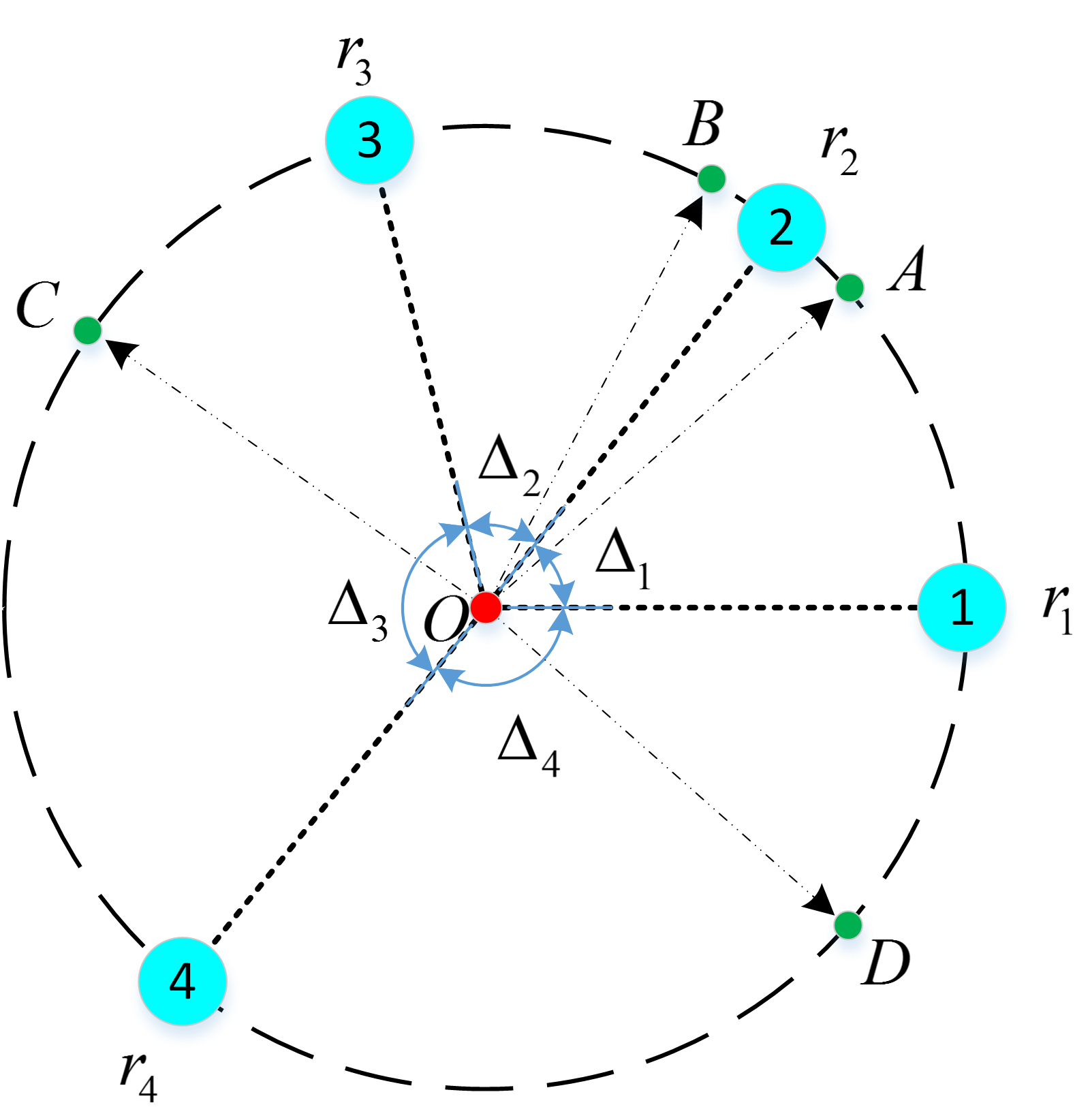}
		 \caption{The interpretation of the formation guideline.}
		 \label{fig:utility}
		\end{minipage}
	\end{figure}
%
%
% % fig 1a and fig 1b
%	\begin{figure}[htb]
%		\centering
%		\begin{subfigure}[b]{0.49\columnwidth}
%			\centering
%			\includegraphics[width=\linewidth]{figures/cylin.png}
%			\caption{The body reference frame with the target $S$ as the origin.}
%			\label{circum_pic1}
%		\end{subfigure}
%		\begin{subfigure}[b]{0.49\columnwidth}
%			\centering
%			\includegraphics[width=\linewidth]{figures/cylin_xy.png}
%			\caption{Robots' projections and the target $S$ on the $X S Y$ plane.}
%			\label{circum_pic2}
%		\end{subfigure}
%		\caption{The body reference frame and the cylindrical coordinate system for the circumnavigation control problem.}
%		\label{problem}
%	\end{figure}

\section{Problem Formulation} \label{sec2}

The research question is that a group of $n \; (n\ge 2)$ mobile robots, denoted by $r_{i} ,\; i=1,...,n$, encircle a target in 3D space with dynamic spacing on a circular formation. Note that $r_i$ is only used to represent the $i$th robot for convenience of narration; it does not correspond to any physical quantities. Suppose each mobile robot is modelled by a 3D kinematic point:
	\begin{equation} \label{eq1} 
	\dot{{\it p}}_{i} (t)={\it u}_{i} (t),\; \; i=1,...,n,
	\end{equation} 
where $u_{i} (t)$ is the control input to the robot $r_{i} $ and $p_{i} (t)\in {\mathbb{R}}^{3} $ is its position in the world reference frame $\mathcal{W}$. In this problem, robots are required to maintain on the same plane with the encircled target which is modelled by another kinematic point. Therefore, a (target) body reference frame $\mathcal{B}$ centred at the target $S$ is introduced (see Fig.~\ref{circum_pic1}). In addition, the  cylindrical coordinate system is preferred to the commonly used Cartesian coordinate system since the former itself embodies three elements of interest: the distance between the projection of the robot on the $X S Y $ plane to the target ($\rho$), the height relative to the $X S Y$ plane ($z$) and the angle between the $X$-axis and the line joining the projection of the robot on the $X S Y $ plane with the target ($\varphi $). The cylindrical coordinates for $r_{i}$ is denoted by $q_{i} =(\rho_{i} ,\varphi_{i} ,z_{i} )^{T} $ (see Fig. \ref{fg1}). To relate the cylindrical coordinates with the Cartesian coordinates, a vector function is defined as $
	{\it q}({\it p})=(\rho ({\it p}),\varphi ({\it p}),z({\it p}))^{T},
$
where ${\it p}\in {\mathbb{R}}^{3} $ is a generic vector with components $p_{x} ,p_{y} ,p_{z}$. $\rho ( p)=\sqrt{p_{x}^{2} +p_{y}^{2} } $, $\varphi (p)={\rm tan}^{{\rm -1}} (p_{y} /p_{x} )$ and $z(p)=p_{z}$.
%	%%
%	\begin{equation} \label{eq3} 
%	\rho ({\it p})=\sqrt{p_{x}^{2} +p_{y}^{2} } ,
%	\end{equation}
%	%%
%	\begin{equation} \label{eq4} 
%	\varphi ({\it p})={\rm tan}^{{\rm -1}} (p_{y} /p_{x} ) ,
%	\end{equation} 
%	%%
%	\begin{equation} \label{eq5} 
%	z({\it p})=p_{z}.
%	\end{equation} 
%	%%
Note that $\varphi \in [0,2\pi)$. The Jacobian matrix of the vector function will be used later, which is
	$
	J=\frac{\partial {\it q}}{\partial {\it p}^{T} } =
	\left[
	\begin{array}{ccc}
	 {\frac{p_{x} }{\sqrt{p_{x}^{2} +p_{y}^{2} } } }  & {\frac{p_{y} }{\sqrt{p_{x}^{2} +p_{y}^{2} } } } & {0} \\ 
	 {\frac{-p_{y} }{p_{x}^{2} +p_{y}^{2} } } & {\frac{p_{x} }{p_{x}^{2} +p_{y}^{2} } } & {0} \\ 
	 {0} & {0} & {1} 
	 \end{array}
	 \right].
	$
For better analysis, we label the robots in the counter-clockwise direction according to their initial (angular) positions ($\varphi_i$) in $\mathcal{B}$ as shown in Fig.~\ref{circum_pic2}. Note that the subscript $i^{-}$ and $i^{+}$ represent the indices of the neighboring robots of $r_i$ in the clockwise and counter-clockwise direction respectively. Especially, if $i=n$,  $i^{+} =1$, and if $i=1$, $i^{-}=n$. $\Delta _{i} > 0$ represents the difference between the angular positions of $r_{i^{+}}$ and that of $r_{i}$. In particular,
	\begin{equation} \label{eq7}
	\Delta_i =
	\begin{dcases}
	 \varphi_{i^+} -\varphi_i ,	  & i=1,\dots,n-1,\\ 
	 \varphi_1 - \varphi_n +2\pi, & i=n. 
	\end{dcases}
	\end{equation}
Also note that 
	\begin{equation} \label{eq8} 
	\sum_{i=1}^{n}\Delta_{i} = 2\pi,
	\end{equation} 

Before giving the definition of the circumnavigation problem with dynamic spacing, we propose the concept of \textit{utility}. Let $\mathbb{R}_{\ge 0}$ denote the field of non-negative real numbers.

\begin{definition}[Utility]
In a heterogeneous multi-robot system, given different kinds of robots, a robot's utility $\mu(t) \in \mathbb{R}_{\ge 0}$ is determined by a given criterion (such as its maximum movement speed). The utility reflects the weight of the robot in the circumnavigation process at time $t$.
\end{definition}

For example, suppose a robot's maximum movement speed is the criterion. Let $\mu_i(t)=\frac{v_{mi}(t)}{v_{M}}, i=1,\dots,n$, where $v_{mi}(t)$ is the maximum movement speed of $r_i$ and $v_{M}$ is the possible greatest movement speed in the heterogeneous multi-robot system. Then $\mu_i(t) \in [0, ~1], i=1,\dots,n$. When $\mu_i(t) = 0$, the robot $r_i$ cannot continue the circumnavigation process with other robots. In this case, its neighboring robots will neglect its role in the circumnavigation process. $\mu_i(t)$ will increase or decrease due to the enhancement or damage of the robot's locomotion capabilities. To explain directly how utilities are utilized to enable dynamic spacing among robots, we simply regard the utility of a robot to be proportional to its maximum movement speed. For simplicity of writing, the symbol $t$ is neglected from $\mu$ unless it causes confusion. The circumnavigation control problem based on utilities is defined as follows:

\begin{definition}[Circumnavigation Control Problem Based on Utilities]
In a heterogeneous multi-robot system composed of $n\; (n\ge 2)$ mobile robots, each of the robot's dynamics are modelled by \eqref{eq1}. Suppose $f_i :\mathbb{R}^{n+1} \rightarrow (0, 2 \pi), 
\; i=1,\dots,n,$ is a smooth function of time and the utilities of robots, which maps utilities to the final holistic formation spacing. Assume $\mathop{\lim }\limits_{t\to \infty } f_i(t, \mu_1,\dots,\mu_n)$ exists, the circumnavigation control problem based on utilities is to seek control laws satisfying the following asymptotic conditions:
	\begin{align}
	\mathop{\lim }\limits_{t \to \infty } \rho_{i} (t) &= \rho^{*} \label{equti1}  \\
	\mathop{\lim }\limits_{t\to \infty } \Delta _{i} (t) &=　\mathop{\lim }\limits_{t\to \infty } f_i \label{equti2} \\
	\mathop{\lim }\limits_{t\to \infty } \dot{\varphi }_{i} (t)　&=\omega ^{*}  \label{equti3}\\
	\mathop{\lim }\limits_{t\to \infty } z_{i} (t)　&=z^* \label{equti4}
	\end{align}
for $i=1,...,n$. Here, $\mu_i >0$, $\rho^{*}>0$, $\omega^{*} \in \mathbb{R}$ and $z^{*} \in \mathbb{R}$.  $\rho^{*}$, $\omega^{*}$ and $z^{*}$ denote the circumnavigation radius, the angular speed and the circumnavigation height respectively. 
\end{definition}

In this paper, it is required that all robots and the target remain in the same plane in the end. Therefore, the default value of $z^{*}$ is 0. However, $z^*$ can be different for different robots.

Equation \eqref{equti2} manifests that the final formation spacing is not specified manually as proposed in \cite{wang2013} or \cite{weijia2017}, but instead, it is determined by the $f_i$ function, which will be referred to as $f$ function for simplicity. Note that $f_i$ function depends on the utilities of other robots instead of calculating by each robot alone. The advantage of eliciting the $f$ function is that the spacing among robots can be dynamically adjusted corresponding to the variations of robots' utilities. 
%
%	\begin{figure}[htb]
%	\centering
%	\includegraphics[width=.50\columnwidth]{figures/utility2}
%	\caption{The physics background of the formation guideline in this paper.}
%	\label{fig:utility}
%	\end{figure}
%

%	\begin{figure*}[!t]
%	\normalsize
%	\setcounter{mytempeqncnt}{\value{equation}}
%	\setcounter{equation}{20}
%		\begin{equation}
%		\label{eq24} 
%		A=\left[
%		\begin{array}{ccccccc}
%		{0} & {\frac{\mu_n + \mu_1 }{\mu_2 + 2 \mu_1 + \mu_n} } & {0} & {\ldots } & {0} & {0} & {\frac{\mu_1 + \mu_2 }{\mu_2 + 2 \mu_1 + \mu_n } } \\ 
%		{\frac{\mu_2 + \mu_3 }{\mu_3 + 2 \mu_2 + \mu_1} } & {0} & {\frac{\mu_1 + \mu_2 }{\mu_3 + 2 \mu_2 + \mu_1 } } & {\ldots } & {0} & {0} & {0} \\ 
%		{\vdots } & {\vdots } & {\vdots } & {\vdots } & {\vdots } & {\vdots } & {\vdots } \\  
%		{\frac{\mu_{n-1} + \mu_n }{\mu_1 + 2 \mu_n + \mu_{n-1} } } & {0} & {0} & {\ldots } & {0} & {\frac{\mu_n + \mu_1 }{\mu_1 + 2 \mu_n + \mu_{n-1} } } & {0}
%		\end{array}
%		\right].
%		\end{equation}
%	\setcounter{equation}{\value{mytempeqncnt}}
%	\hrulefill
%	\vspace*{4pt}
%	\end{figure*}
	
The expression of the $f$ function is determined by a \textit{formation guideline}. It is proposed under specific physics background representing the relationship between the utilities of robots and the final formation spacing. In this paper, we suppose that multiple heterogeneous robots circumnavigate a target and try to prevent it from fleeing. In Fig. \ref{fig:utility}, four robots $r_1,\dots,r_4$ rotate around a target denoted by $O$. Suppose that the target is intelligent enough to determine the \textit{best fleeing points} denoted by $A$, $B$, $C$ and $D$ in the figure. Obviously, the best fleeing points are related to the utility (i.e., the maximum movement speeds) of robots. The position of $A$, for instance, is calculated by $\angle AOr_2 = \frac{\mu_2}{\mu_1+\mu_2}$. We also suppose that the probability of capturing the target by a robot is inversely proportional to the time spent on moving from its initial position along the circular trajectory at its maximum speed to the best fleeing point. Therefore, the formation guideline can be defined as
	
	\begin{formation} \label{fg1}
		In the final circumnavigation formation formed by robots, when the target tries to escape via any of the best fleeing point, the two robots adjacent to the best fleeing point have the same probability of capturing the target.
	\end{formation}
	
To understand the above formation guideline, taking Fig. \ref{fig:utility} for example, it means the travelling time for $r_1$ and $r_2$ to arrive at the best fleeing point $A$ along the circular trajectory at their maximum speeds (i.e., $\mu_1$ and $\mu_2$ resp.) is the same, or the travelling time for $r_2$ to arrive at $A$ or $B$ along the circular trajectory at its maximum speed (i.e., $\mu_2$) is identical, and hence, the probability of capturing the target is equal. Following this, it can be derived that for $i=1,\dots,n$, $
		\frac{\mu_i}{\mu_i + \mu_{i^+}} \Delta_{i} = \frac{\mu_i}{\mu_i + \mu_{i^-}} \Delta_{i-}.
$
%
%To elaborate that the probabilities of capturing the target are the same under Formation Guideline \ref{fg1}, we define $t_i, i=1,\dots,b$ to be the traveling time for the robot $r_i$. Then according to \eqref{eq4.9} it is obvious that
%%
%	\begin{equation}
%	t_i = \frac{1}{\mu_i + \mu_{i^+}} \Delta_{i} = \frac{1}{\mu_i + \mu_{i^-}} \Delta_{i-},
%	\end{equation}
%%
%which means $t_i = t_j,~i=1,\dots,n, ~j=1,\dots,n, ~i \ne j$. That is, the probabilites are equal and \eqref{eq4.9} satisfies Formation Guideline \ref{fg1}. 
%
According to this equation, the relationship between the final desired formation spacing and the utilities is
	$
	\Delta_{1}:\Delta_{2}:\dots:\Delta_{n} = (\mu_1+\mu_2):(\mu_2+\mu_3):\dots:(\mu_n+\mu_1).
	$ 
Therefore, given $\mu_1,\dots,\mu_n$, the formation spacing can be determined, and the $f$ function is expressed as follows:
	\begin{equation}　\label{eqf1}
	f_i(t, \mu_{1},\dots,\mu_{n}) = \frac{\mu_i + \mu_{i^+}}{\sum_{k=1}^{n} \mu_k}\pi.
	\end{equation}

Similarly, another formation guideline can be defined as follows:

	\begin{formation} \label{fg2}
	In the final circumnavigation formation formed by robots, the spacing among robots is proportional to their utilities.
	\end{formation}
	
Formation Guideline \ref{fg2} can be interpreted in this way: as the utility of a robot increases, its complementary effects on its neighboring robots increase, and therefore the spacing between them should be increased. Taking Fig. \ref{fig:utility} for example, the formation guideline indicates that, for $i=1,\dots,n$, 
	$
	\frac{\Delta_{i}}{\mu_i}=\frac{\Delta_{i^-}}{\mu_{i^-}}.
	$
Accordingly, given $\mu_1,\dots,\mu_n$, the final desired formation spacing is determined. Hence the $f$ function is
	$
	f_i(t, \mu_{1},\dots,\mu_{n}) = \frac{2 \mu_i}{\sum_{k=1}^{n} \mu_k}\pi.
	$

\begin{remark}
	Note that formation guidelines only reflect the relationship between the utilities of robots and the final formation spacing; it does not determine the utilities of robots. 
\end{remark}

\section{Utility-based Circumnavigation Control Algorithm} \label{sec3}

First we define a rotational matrix $R_{b}$, which is the representation of the body reference frame $\mathcal{B}$ with respect to the world reference frame $\mathcal{W}$. Therefore, the following equation calculates the cylindrical coordinates of $r_i$ in the frame $\mathcal{B}$: $
	q_{i} ={\it q}(R_{b}^{T} ({\it p}_{i} -{\it p}_{b} )) ,
$
where ${\it p}_{i} $ and ${\it p}_{b} $ are the Cartesian coordinates of $r_i$ and the target in the frame $\mathcal{W}$ respectively. Then the derivative of  $q_i$ is the dynamics of robots in the cylindrical coordinates, which is $
	\dot{{\it q}}_{i} ={\it J}_{i} [\dot{R}_{b}^{T} ({\it p}_{i} -{\it p}_{b} )+R_{b}^{T} (\dot{{\it p}}_{i} -\dot{{\it p}}_{b} )] ,
$
where ${\it J}_{i} $ is the Jacobian matrix as shown in Section \ref{sec2}, i.e. ${\it J}_{i} =\frac{\partial {\it q}}{\partial {p}^{T} } \left|\begin{array}{c} {} \\ {{p}=R_{b}^{T} ({\it p}_{i} -{\it p}_{b} )} \end{array}\right. $. Note that $\det (J)=\frac{1}{\sqrt{p_{x}^{2} +p_{y}^{2}}}$ as long as $p_{x}^{2} +p_{y}^{2} \ne 0$. This means ${\it J}_{i} $ is invertible as long as the distance between $r_i$ and the target is non-zero. This condition can always be guaranteed since the initial positions of the robots and the target do not coincide, and by designing appropriate control algorithms, the distance can be guaranteed to be non-zero all the time. By letting 
	\begin{equation} \label{eq15} 
	{\it u}_{i} =\dot{{\it p}}_{i} =\dot{{\it p}}_{b} +R_{b} ({\it J}_{i} ^{-1} {\it v}_{i} -\dot{R}_{b}^{T} ({\it p}_{i} -{\it p}_{b} )) ,
	\end{equation} 
we can switch our focus to the new control input in the cylindrical coordinates ${\it v}_{i} =\dot{{\it q}}_{i} =(\dot{\rho }_{i} ,\dot{\varphi }_{i} ,\dot{z}_{i} )^{T}$ \cite{franchi2016}. The advantage of transforming to this control input is that we can control $\rho _{i} $, $\varphi _{i}$ and $z_{i}$ separately, which are the three main variables in the circumnavigation problem. 

\textbf{Notations.} 
For positive integers $m$ and $n$, $M_{n}$ and $M_{m \times n}$ are a set of all $n$-by-$n$ and $m$-by-$n$ real matrices. If all the entries in a matrix is nonnegative, this matrix is called nonnegative. We denote $I_{d}$ as the $d \times d$ identity matrix. $ \boldsymbol{1}$ and $\boldsymbol{0}$ are vectors of all 1's or 0's of suitable dimensions respectively. The underlying directed graph (or digraph) of a nonnegative matrix $M\in M_{n}$, denoted by $\mathcal{G}(M)$, is the directed graph with the vertex set $\{ v_{i} \} ,i\in \{ 1,...,n\} $, such that there is a directed edge in $\mathcal{G}(M)$ from $v_{j}$ to $v_{i}$ if and only if $m_{ij} \ne 0$ \cite{horn2012}. A directed graph is called strongly connected if for every pair of vertices, there is a directed path between them \cite{mesbahi2010}.  The following is a preliminary result related to any strongly connected digraph. 

\begin{lemma}[{Theorem 3 of \cite{olfati2004}}]
Assume $G$ is a strongly connected digraph with Laplacian $L$ satisfying $Lw_{r} =\boldsymbol{0}$, $w_{l} ^{T} L=\boldsymbol{0}$ and $w_{l} ^{T} w_{r} =1$. Then
	$
	R=\mathop{\lim }\limits_{t\to \infty } \exp (-Lt)=w_{r} w_{l} ^{T} \in M_{n} .
	$
\end{lemma}

%\begin{figure*}[!t]
%\normalsize
%\setcounter{mytempeqncnt2}{\value{equation}}
%\setcounter{equation}{27}
%	\begin{equation} \label{M_delta}
%		M_{\Delta}=\left[
%		\begin{array}{cccc}
%{\frac{-(\mu_2 + \mu_3) }{\mu_3 + 2 \mu_2 + \mu_1}+\frac{-(\mu_n + \mu_1) }{\mu_2 + 2 \mu_1 + \mu_n} } & {\frac{\mu_1 + \mu_2 }{\mu_3 + 2 \mu_2 + \mu_1 }}  & {\ldots } & {\frac{\mu_1 + \mu_2 }{\mu_2 + 2 \mu_1 + \mu_n } } \\ 
%		{\frac{\mu_2 + \mu_3 }{\mu_3 + 2 \mu_2 + \mu_1 } } & {\frac{-(\mu_3 + \mu_4) }{\mu_4 + 2 \mu_3 + \mu_2}+\frac{-(\mu_1 + \mu_2) }{\mu_3 + 2 \mu_2 + \mu_1} } &  {\ldots } & {0} \\ 
%		{\vdots } & {\vdots } & {\vdots } & {\vdots } \\ 
%		{\frac{\mu_n + \mu_1 }{\mu_2 + 2 \mu_1 + \mu_n } } & {0} & {\ldots } & {\frac{-(\mu_1 + \mu_2) }{\mu_2 + 2 \mu_1 + \mu_n}+\frac{-(\mu_{n-1} + \mu_n) }{\mu_1 + 2 \mu_n + \mu_{n-1}} } 
%		\end{array}
%		\right].
%	\end{equation}
%\setcounter{equation}{\value{mytempeqncnt2}}
%\hrulefill
%\vspace*{4pt}
%\end{figure*}

\begin{theorem} \label{theorem1}
Consider a multi-robot system with robot dynamics described by \eqref{eq1} and \eqref{eq15}, by introducing the control input ${\it v}_{i} =\dot{{\it q}}_{i} =(\dot{\rho }_{i} ,\dot{\varphi }_{i} ,\dot{z}_{i} )^{T}$ into \eqref{eq15}, where
	\begin{align}
	\dot{\rho }_{i}    &=k_{\rho } (\rho ^{*} -\rho _{i} ),                              \label{eq17}  \\
	\dot{z}_{i}        &=-k_{z} z_{i},                                                   \label{eq18}  \\
	\dot{\varphi }_{i} &=\omega ^{*} +k_{\varphi } (\bar{\varphi }_{i} -\varphi _{i} ).  \label{eq19}  
	\end{align}
Note that $k_{\rho } $, $k_{z} $ and $k_{\varphi }$ are positive gains, and 
	\begin{equation} \label{eq4.13}
	\bar{\varphi }_{i} =
	\begin{dcases}
		\varphi _{i^-} +\frac{\mu_{i^-}+\mu_i}{\mu_{i^+}+2\mu_i+\mu_{i^-}}　(\Delta_i +\Delta _{i^-}), \; i=2,3,\dots,n, \\ 
		\varphi_{i^-} +\frac{\mu_{i^-}+\mu_i}{\mu_{i^+}+2\mu_i+\mu_{i^-}}(\Delta_i +\Delta_{i^-}) -2\pi, \; i=1,
	\end{dcases} 
	\end{equation}
where $\mu_i, i=1,\dots,n,$ is the utility of the robot $r_i$ and it is piecewise constant. If the $f$ function is shown as \eqref{eqf1} (Formation Guideline \ref{fg1}), the circumnavigation control problem based on utilities encoded by \eqref{equti1}, \eqref{equti2}, \eqref{equti3} and \eqref{equti4} can be solved with exponential convergence speed. 
\end{theorem}
		
\begin{proof}
	It is obvious that  \eqref{eq17} and  \eqref{eq18} do not rely on the states of other robots, and they are basically P control laws with reference input $\rho ^{*}$ and 0 respectively. So according to the classical control theory, $\rho _{i} $ and $z_{i} $ will converge exponentially to $\rho ^{*} $ and 0 respectively. 
	
	Since $\mu_i, ~i=1,\dots,n,$ is piecewise constant, it is obvious that $\mathop{\lim }\limits_{t\to \infty } f_i$ exists. We define $\bar{\varphi}=[\bar{\varphi }_{1} \;...\; \bar{\varphi }_{n} ]^{T} $ and $\varphi =[\varphi _{1} \;...\; \varphi _{n} ]^{T} $, so  \eqref{eq19} and  \eqref{eq4.13} can be written into compact forms respectively as follows:
		\begin{equation} \label{eq22} 
		\dot{\varphi }=\omega ^{*} \boldsymbol{1}+k_{\varphi } (\bar{\varphi }-\varphi ),
		\end{equation} 
		\begin{equation} \label{eq23} 
		\bar{\varphi }=A\varphi +b,
		\end{equation} 
		%Account for the double column equations here.
		%\addtocounter{equation}{1}%
		%
		\begin{equation}
		\label{eq24} 
		A=\left[
		\begin{array}{ccccccc}
		{0} & {\frac{\mu_n + \mu_1 }{\mu_2 + 2 \mu_1 + \mu_n} } & {0} & {\ldots } & {0} & {0} & {\frac{\mu_1 + \mu_2 }{\mu_2 + 2 \mu_1 + \mu_n } } \\ 
		{\frac{\mu_2 + \mu_3 }{\mu_3 + 2 \mu_2 + \mu_1} } & {0} & {\frac{\mu_1 + \mu_2 }{\mu_3 + 2 \mu_2 + \mu_1 } } & {\ldots } & {0} & {0} & {0} \\ 
		{\vdots } & {\vdots } & {\vdots } & {\vdots } & {\vdots } & {\vdots } & {\vdots } \\  
		{\frac{\mu_{n-1} + \mu_n }{\mu_1 + 2 \mu_n + \mu_{n-1} } } & {0} & {0} & {\ldots } & {0} & {\frac{\mu_n + \mu_1 }{\mu_1 + 2 \mu_n + \mu_{n-1} } } & {0}
		\end{array}
		\right]
		\end{equation}
		where $A\in M_{n}$ is shown as \eqref{eq24}, and $b=2\pi \left[
		\begin{array}{ccccc} {\frac{-(\mu_1 + \mu_2) }{\mu_2 + 2 \mu_1 + \mu_n } } & {0} & {\ldots } & {0} & {\frac{\mu_{n-1} + \mu_n }{\mu_1 +2 \mu_n + \mu_{n-1} } } \end{array}
		\right]^T.
		$ 
	During each time period where $\mu_i$ is constant, $A$ and $b$ are constant matrix and vector respectively. Note that matrix $A$ is a row stochastic matrix \cite{horn2012} and furthermore, it could be considered as the adjacency matrix \cite{mesbahi2010} corresponding to a weighted directed ring denoted by $\mathcal{G}(A)$. It can be readily verified that $\mathcal{G}(A)$ is strongly connected. Next we define the \textit{error signal} as
		\begin{equation} \label{eq26} 
		e_{\varphi } =\bar{\varphi }-\varphi =(A-I_{n} )\varphi +b=-L_{p} \varphi +b, 
		\end{equation} 
	where $L_{p} =I_{n} -A$, which is the Laplacian matrix of $\mathcal{G}(A)$. Since $L_p$ is constant at each time period, the derivative of $e_{\varphi } $ is $\dot{e}_{\varphi } =-L_{p} \dot{\varphi }$. By substituting  \eqref{eq22} and  \eqref{eq26} into this equation, we further obtain the \textit{error dynamics} as
		\begin{equation} \label{eq27} 
		\dot{e}_{\varphi } =-\omega ^{*} L_{p} \boldsymbol{1}-k_{\varphi } L_{p} e_{\varphi } =-k_{\varphi } L_{p} e_{\varphi } .
		\end{equation} 
	Note that \textbf{1} is the right eigenvector associated with the zero eigenvalue of $L_{p} $, so $-\omega ^{*} L_{p} \boldsymbol{1}=0$. The solution to  \eqref{eq27} is $e_{\varphi } (t)=\exp (-k_{\varphi } L_{p} t)e_{\varphi }(0)$. According to Lemma 1 and also note that $k_{\varphi } >0$ only affects the convergence speed but not the convergence value, we have $\mathop{\lim }\limits_{t\to \infty } e_{\varphi } (t)=w_{r} w_{l} ^{T} e_{\varphi}(0)$, where $L_{p} w_{r} =\boldsymbol{0},\; w_{l} ^{T} L_{p} =\boldsymbol{0}$ and $w_{l} ^{T} w_{r} =1$. By substituting  \eqref{eq26} into this equation, we obtain the following:
		\begin{equation} \label{eq28} 
		\mathop{\lim }\limits_{t\to \infty } e_{\varphi } (t)=w_{r} (-w_{l} ^{T} L_{p} \varphi +w_{l} ^{T} b)=w_{l} ^{T} bw_{r}.
		\end{equation}
	Let $w_{r} =\boldsymbol{1} $ and $
		w_{l} =\frac{w_{L}}{\sum_{w_L}},
	$
	where the $i^{th}$ entry of $w_{L} $ is
	\[
		\left[ w_{L_i} =(\mu_{i^+} + 2 \mu_i + \mu_{i^-})\mathop{\prod }\limits_{j=1,j\ne i,i^{-} }^{n} (\mu_{j} + \mu_{j^+}) \right],
	\]
	and $\sum_{w_L}=\sum_{i=1}^{n}w_{L_i}$. It can be easily verified that $w_{l}^{T} $and $w_{r} $ are the left and right eigenvector of the Laplacian matrix $L_{p} $ associated with the zero eigenvalue respectively, and $w_{l} ^{T} w_{r} =1$. Therefore,  \eqref{eq28} becomes $\mathop{\lim }\limits_{t\to \infty } e_{\varphi } (t)=\boldsymbol{0}$, which means that the difference between the desired angular position and the actual angular position of each robot vanishes to zero exponentially, or 
	$
		\mathop{\lim }\limits_{t\to \infty } \varphi (t)=\mathop{\lim }\limits_{t\to \infty } \bar{\varphi }(t). 
	$
	According to  \eqref{eq22}, the circumnavigation speed of each robot converges to the desired angular speed $\omega ^{*} $. In addition, under this condition, $\bar{\varphi }_{i} $ is replaced by $\varphi _{i} $ in  \eqref{eq4.13} and therefore, for robots with indices $i=2,...,n$, the equation $\varphi _{i} =\varphi _{i^-} +\frac{\mu_{i^-}+\mu_i}{\mu_{i^+}+2\mu_i+\mu_{i^-}}　(\Delta_i +\Delta _{i^-}) $ further becomes 
	$
		\frac{\Delta _{i} }{\Delta _{i^{-} } } =\frac{\mu_{i} + \mu_{i^+} }{\mu_{i} + \mu_{i^-}}. 
	$
	This means a sequence of equations $\frac{\Delta _{n} }{\Delta _{n-1} } =\frac{\mu_n + \mu_1 }{\mu_{n-1} + \mu_n } ,...,\frac{\Delta _{2} }{\Delta _{1} } =\frac{\mu_2 + \mu_3 }{\mu_1 + \mu_2 } $. Assuming $\Delta _{1} =k(\mu_1 + \mu_2) ,\; k\ne 0$, we have $\Delta _{i} =k (\mu_{i} + \mu_{i^+}), i=2,...,n$. According to \eqref{eq8}, it follows that $2 k \sum_{i=1}^{n} \mu_i = 2 \pi$, and hence $k=\pi / \sum_{i=1}^{n} \mu_i$. Therefore, $\Delta _{i} =(\mu_{i} + \mu_{i^+}) \pi / \sum_{i=1}^{n} \mu_i = f_i(t, \mu_{1},\dots,\mu_{n}) ,\; i=1,...,n$. So the formation spacing expressed by \eqref{equti2} and \eqref{eqf1} can be achieved.      
\end{proof}
	
\begin{remark} 
Since \eqref{eq17}, \eqref{eq18}, \eqref{eq19} and \eqref{eq27} typically admit a linear system, the convergence is global and exponential. In fact, for the convergence of $e_\varphi$, a Lyapunov function can be defined as $V(e_\varphi)=e_\varphi^T P e_\varphi$, where $P=\text{diag}\{w_l\}$, so the global and exponential convergence can also be proved using the Lyapunov theorem. In addition, if $\mu_1+\mu_2=\dots=\mu_{n-1}+\mu_n=\mu_n+\mu_1$, \eqref{eqf1} becomes $	f_i = 2 \pi / n$, that is, the spacing among robots is equal. Then  \eqref{eq4.13} becomes
	\begin{equation} \label{33} 
	\left\{\begin{array}{l} {\bar{\varphi }_{1} =\frac{\varphi _{n} +\varphi _{2} -2\pi }{2} }, \\ {\bar{\varphi }_{i} =\frac{\varphi _{i-1} +\varphi _{i+1} }{2} ,\; i=2,...,n-1}, \\ 
	{\bar{\varphi }_{n} =\frac{\varphi _{n-1} +\varphi _{1} +2\pi }{2}. } \end{array}\right.  
	\end{equation} 
Combined with  \eqref{eq17},  \eqref{eq18} and  \eqref{eq19}, they can solve circumnavigation problem with equal spacing. 
\end{remark}

\begin{remark}
It is interesting to note that when there are only two robots, i.e., $ n=2 $, the subscripts of robots satisfy $ i^- = i^+$.  Equation~\eqref{eq4.13} is simplified to $\bar{\varphi }_{1} = \varphi _{2} - \pi, ~ \bar{\varphi }_{2} = \varphi _{1} + \pi$, and \eqref{eqf1} becomes $f_i=\pi, ~i=1,2$. This indicates that the formation spacing is fixed no matter how robots' utilities change (except 0); these two robots will always position on the ends of the diameter of the circular trajectory. Furthermore, $A$ and $b$ degrade to $ A=\left[\begin{array}{cc} {0} & {1} \\ {1} & {0} \end{array}\right] $ and $ b=\left[\begin{array}{c} {-\pi } \\ {\pi } \end{array}\right] $ respectively.
\end{remark}

\begin{remark}
In the definition of circumnavigation control problem based on utilities, \eqref{equti2} contains the utilities of all robots. However, it can be seen from \eqref{eq4.13} that each robot only needs to obtain the utilities of its two neighboring robots. Moreover, it should be noted that robots do not know what the holistic expected formation is; the actual formation (or spacing) among robots adapt dynamically to the variations of the local utilities of neighboring robots. In addition, when a robot joins or leaves the formation, according to \eqref{eq19} and \eqref{eq4.13}, the spacing among robots will adjust dynamically through local update of the utilities of neighboring robots. To sum up, the utility-based circumnavigation control algorithm does not rely on the number of robots, and it is able to dynamically adjust the formation spacing dependent on the change of utilities. The control algorithm is distributed, and in this way, it achieves the global aim described by \eqref{equti2}.
\end{remark}

\begin{remark} \label{remark5}
	When $\mu_\theta=0$, the robot $r_\theta$ has quitted from the circumnavigation process, and therefore the communication topology has changed. The change of communication topology means the indices of the neighboring robots alter accordingly. When $\mu_2=0$, for example, the neighboring robots of $r_3$ change from $r_2$ and $r_4$ to $r_1$ and $r_4$. In this way, the circumnavigation control algorithm based on utilities can well adapt to the cases where there are local variations on utilities or where robots join or quit from the formation. The formation spacing can adjust dynamically based on the selected formation guideline, achieving distributed formation reconfiguration.
\end{remark}

Another problem that is worth considering is whether robots preserve their initial orders during the whole circumnavigation process. This means any robot will not overtake or be overtaken by its neighbours, which guarantees that they will not collide with each other if they are regarded as mass points. Before introducing the next theorem, the definition of a Metzler matrix \cite{minc1988nonnegative} is given. For a real matrix $M = [m_{ij}] \in M_n$, if all its off-diagonal elements are non-negative, i.e., $m_{ij} \ge 0, i \ne j$, $M$ is a Metzler matrix.
	
\begin{theorem} \label{theorem2}
During the circumnavigation process, robots always keep their initial orders in the formation. In other words, $\Delta_i(t) > 0, i=1,\dots,n$, for $t \ge 0$.
\end{theorem}

\begin{proof}
	According to \eqref{eq7}, \eqref{eq19} and \eqref{eq4.13}, for $i=1,\dots,n$, it follows that
	%	\begin{equation}\label{dotDelta}
	%	\dot{\Delta}_i = k_\varphi \left[ \frac{\mu_{i} + \mu_{i^+} }{\mu_{i^{*}} + 2 \mu_{i^+} + \mu_{i}} \Delta_{i^+} -\left( \frac{\mu_{i^+} + \mu_{i^{*}} }{\mu_{i^{*}} + 2 \mu_{i^+} + \mu_{i}}+\frac{\mu_{i^-} + \mu_{i} }{\mu_{i^+} + 2 \mu_{i} + \mu_{i^-}} \right) \Delta_{i} + \frac{\mu_{i} + \mu_{i^+}}{\mu_{i^+} + 2 \mu_{i} + \mu_{i^-}}  \Delta_{i^-} \right],
	%	\end{equation}
	
	% add alignment
		\begin{equation}
		\begin{split}\label{dotDelta}
		\dot{\Delta}_i {} = {} & k_\varphi \left[ \frac{\mu_{i} + \mu_{i^+}}{\mu_{i^{*}} + 2 \mu_{i^+} + \mu_{i} } \Delta_{i^+} - \left( \frac{\mu_{i^+} + \mu_{i^{*}} }{\mu_{i^{*}}+ 2 \mu_{i^+} + \mu_{i}}　\right. \right. \\ &+ \left. \left. \frac{\mu_{i^-} + \mu_{i} }{\mu_{i^+} + 2 \mu_{i} + \mu_{i^-} } \right) \Delta_{i} + \frac{\mu_{i} + \mu_{i^+} }{\mu_{i^+} + 2 \mu_{i} +　\mu_{i^-} } \Delta_{i^-} \right],
		\end{split}
		\end{equation}
	where $i^*$ represents $(i^+)^+$, which is the index of the second adjacent robot for the robot $r_i$ in the counter-clockwise direction. Let $\Delta = [\Delta_1 \dots \Delta_n]^T$, then \eqref{dotDelta} can be rewritten as $
			\dot{\Delta} = k_\varphi M_\Delta \Delta,
	$
		%Account for the double column equations here.
		%\addtocounter{equation}{1}%
		\begin{equation} \label{M_delta}
		\begin{split}
			&M_{\Delta}=\\
			&\left[
			\begin{array}{cccc}
	{\frac{-(\mu_2 + \mu_3) }{\mu_3 + 2 \mu_2 + \mu_1}+\frac{-(\mu_n + \mu_1) }{\mu_2 + 2 \mu_1 + \mu_n} } & {\frac{\mu_1 + \mu_2 }{\mu_3 + 2 \mu_2 + \mu_1 }}  & {\ldots } & {\frac{\mu_1 + \mu_2 }{\mu_2 + 2 \mu_1 + \mu_n } } \\ 
			{\frac{\mu_2 + \mu_3 }{\mu_3 + 2 \mu_2 + \mu_1 } } & {\frac{-(\mu_3 + \mu_4) }{\mu_4 + 2 \mu_3 + \mu_2}+\frac{-(\mu_1 + \mu_2) }{\mu_3 + 2 \mu_2 + \mu_1} } &  {\ldots } & {0} \\ 
			{\vdots } & {\vdots } & {\vdots } & {\vdots } \\ 
			{\frac{\mu_n + \mu_1 }{\mu_2 + 2 \mu_1 + \mu_n } } & {0} & {\ldots } & {\frac{-(\mu_1 + \mu_2) }{\mu_2 + 2 \mu_1 + \mu_n}+\frac{-(\mu_{n-1} + \mu_n) }{\mu_1 + 2 \mu_n + \mu_{n-1}} } 
			\end{array}
			\right].
		\end{split}
		\end{equation}
	 where $M_\Delta$ is shown in \eqref{M_delta}. Therefore, the solution of $\Delta(t)$ is $
			\Delta(t) = \exp(k_\varphi M_\Delta t) \\ \Delta(0).
	$
	Since $M_\Delta$ is a Metlzer matrix, it has been proved that $\exp(k_\varphi M_\Delta t)$ is a non-negative matrix \cite{minc1988nonnegative}. In addition, due to $\Delta(0) > 0$, it follows that $\Delta(t) > 0, ~t \ge 0$, which means that robots always keep their initial orders in the formation.
\end{proof}

\begin{remark}
	The significance of this theorem is that it provides a preliminary result for collision avoidance. In other words, if robots are treated as mass points, then collision will not happen since they always keep their initial orders. For real robots with geometric shape, given sufficiently large spacing, the collision will not happen, but this will need further investigation.
\end{remark}

Similarly, with Formation Guideline \ref{fg2}, the circumnavigation control problem based on utilities can also be solved.

%\begin{prop}
%Consider a multi-robot system with robot dynamics described by \eqref{eq1} and  \eqref{eq15}, by introducing the control input ${\it v}_{i} =\dot{{\it q}}_{i} =(\dot{\rho }_{i} ,\dot{\varphi }_{i} ,\dot{z}_{i} )^{T}$ into \eqref{eq15}, where $\dot{\rho }_{i}$ and $\dot{z}_{i}$ are given by \eqref{eq17} and \eqref{eq18} respectively.
%%
%	\begin{equation} \label{eqf2phi}
%	\dot{\varphi }_{i} =\omega ^{*} +k_{\varphi } (\hat{\varphi }_{i} -\varphi _{i} ),    
%	\end{equation}
%%
%where
%%
%	\begin{equation} \label{eqf2phihat} 
%	\hat{\varphi }_{i} =
%	\begin{dcases}
%		\varphi _{i^-} +\frac{\Delta_i +\Delta _{i^-}}{\mu_i +\mu_{i^-}} \mu_{i^-} ,&i=2,3,\dots,n, \\ 
%		\varphi_{i^-} +\frac{\Delta_i +\Delta_{i^-}}{\mu_i +\mu_{i^-}} \mu_{i^-} -2\pi,  &i=1,
%	\end{dcases} 
%	\end{equation}
%%
%where $\mu_i, i=1,\dots,n,$ is the utility of the robot $r_i$ and it is piecewise constant. If the $f$ function is shown as \eqref{eqf2} (Formation Guideline \ref{fg2}), the circumnavigation control problem based on utilities encoded by \eqref{equti1}, \eqref{equti2}, \eqref{equti3} and \eqref{equti4} can be solved with exponential convergence speed. 
%\end{prop}
%
%\begin{proof}
%The proof is similar to that of Theorem \ref{theorem1}, so it is omitted here.
%\end{proof}
%
%Theorem \ref{theorem2} also applies to the circumnavigation process under Formation Guideline \ref{fg2}.

\section{Experimental Results and Analysis} \label{sec4}

Although it is claimed that formation guidelines correspond to specific physics backgrounds, in the experiment, we do not try to reproduce the specific scenarios. This is because the emphasis here is the stability of the circumnavigation control algorithm based on utilities, and how the global formation spacing reacts dynamically to the variation of the utilities. In the experiments, robots' utilities are supposed to be proportional to its maximum movement speed. However, how the utilities are calculated from the maximum movement speeds is not the interest of the study. Instead, the variation of the utilities are manually specified. Readers can think of an increase in the utilities as an update of robots' locomotion capabilities, while the decrease means the deterioration of performance due to worn-out or damage of robots. The first experiment is a simulation with Simulink and the other is a real-robot experiment using soccer-playing robots.

\subsection{Simulation with Simulink}

In this simulation, four robots are used and Formation Guideline \ref{fg2} is adopted. For demonstrating the dynamic change in formation spacing when robots' utilities vary, the utilities of the robot $r_1$, $r_3$ and $r_4$ remain 1 throughout the whole circumnavigation process, while the utility of the robot $r_2$ varies according to a piecewise constant function. That is $\mu_2=2, ~0 \le t <5; ~\mu_2=0, ~5 \le t <10; ~\mu_2=0.5, ~t \ge 10$.
%
%	\begin{equation} \label{eqmu2_1}
%	\mu_2 =
%	\begin{dcases}
%	2,    & 0 \le t <5,   \\
%	0,    & 5 \le t <10,  \\
%	0.5,  & t \ge 10.
%	\end{dcases}
%	\end{equation}
%
For convenience, the three time ranges are denoted by Stage 1, 2 and 3 respectively. Note that at Stage 2, the utility of the robot $r_2$ is zero, which means it does not continue the circumnavigation process but retreats to a corner. In this experiment, robots' initial positions are randomly chosen. Simulation parameters are $\rho^*=2$ m, $w^*=1$ rad/s and $k_\varphi=k_\rho=k_z=2$. 
	\begin{figure}[!htb]
		\centering %
		\begin{subfigure}[b]{.4\columnwidth}
			\includegraphics[width=\linewidth]{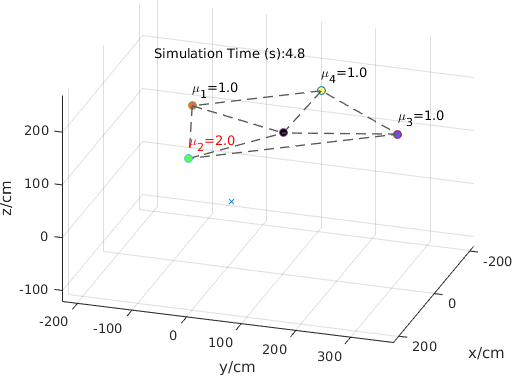}
			\caption{}
			\label{util_sim0}
		\end{subfigure}
		\begin{subfigure}[b]{.4\columnwidth}
			\includegraphics[width=\linewidth]{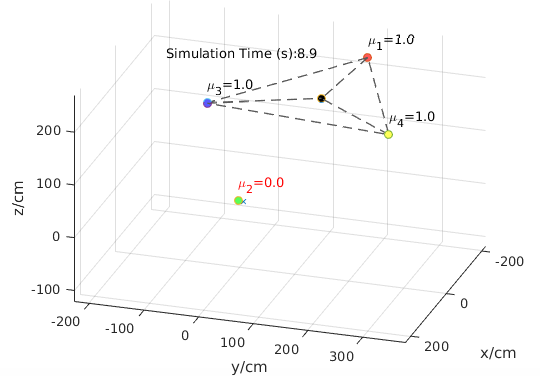}
			\caption{}
			\label{util_sim1}
		\end{subfigure}
		\begin{subfigure}[b]{.4\columnwidth}
			\includegraphics[width=\linewidth]{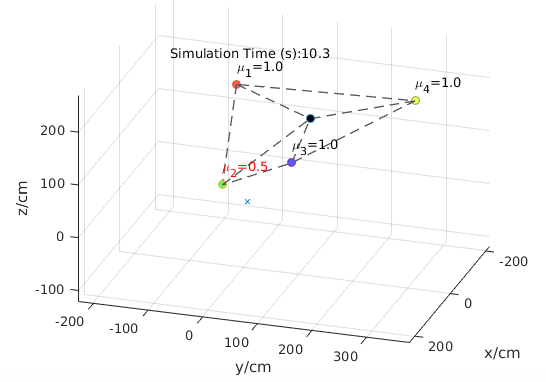}
			\caption{}
			\label{util_sim2}
		\end{subfigure}
		\begin{subfigure}[b]{.4\columnwidth}
			\includegraphics[width=\linewidth]{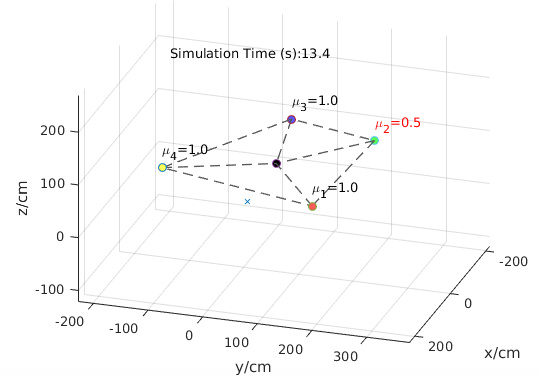}
			\caption{}
			\label{util_sim3}
		\end{subfigure}
		\caption{The simulation with Simulink. (a)-(d) are the circumnavigation process at 4.8 s (Stage 1), 8.9 s (Stage 2), 10.3 s (Stage 3) and 13.4 s (Stage 3) respectively.}
		\label{utility_sim}
	\end{figure}
	\begin{figure}[!htb]
		\centering %
		\begin{subfigure}[b]{.4\columnwidth}
			\includegraphics[width=\linewidth]{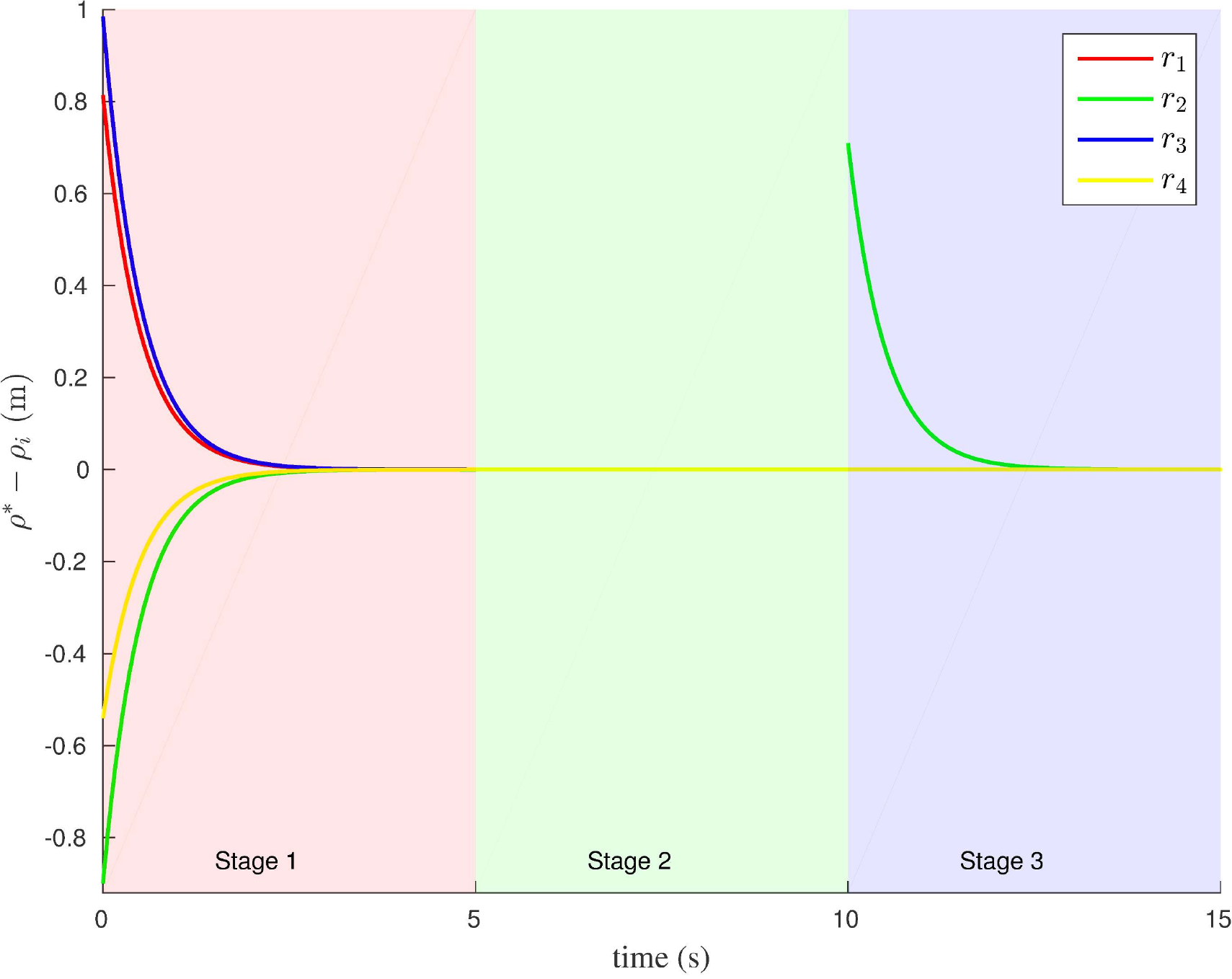}
			\caption{}
			\label{util_sim_rou}
		\end{subfigure}
		\begin{subfigure}[b]{.4\columnwidth}
			\includegraphics[width=\linewidth]{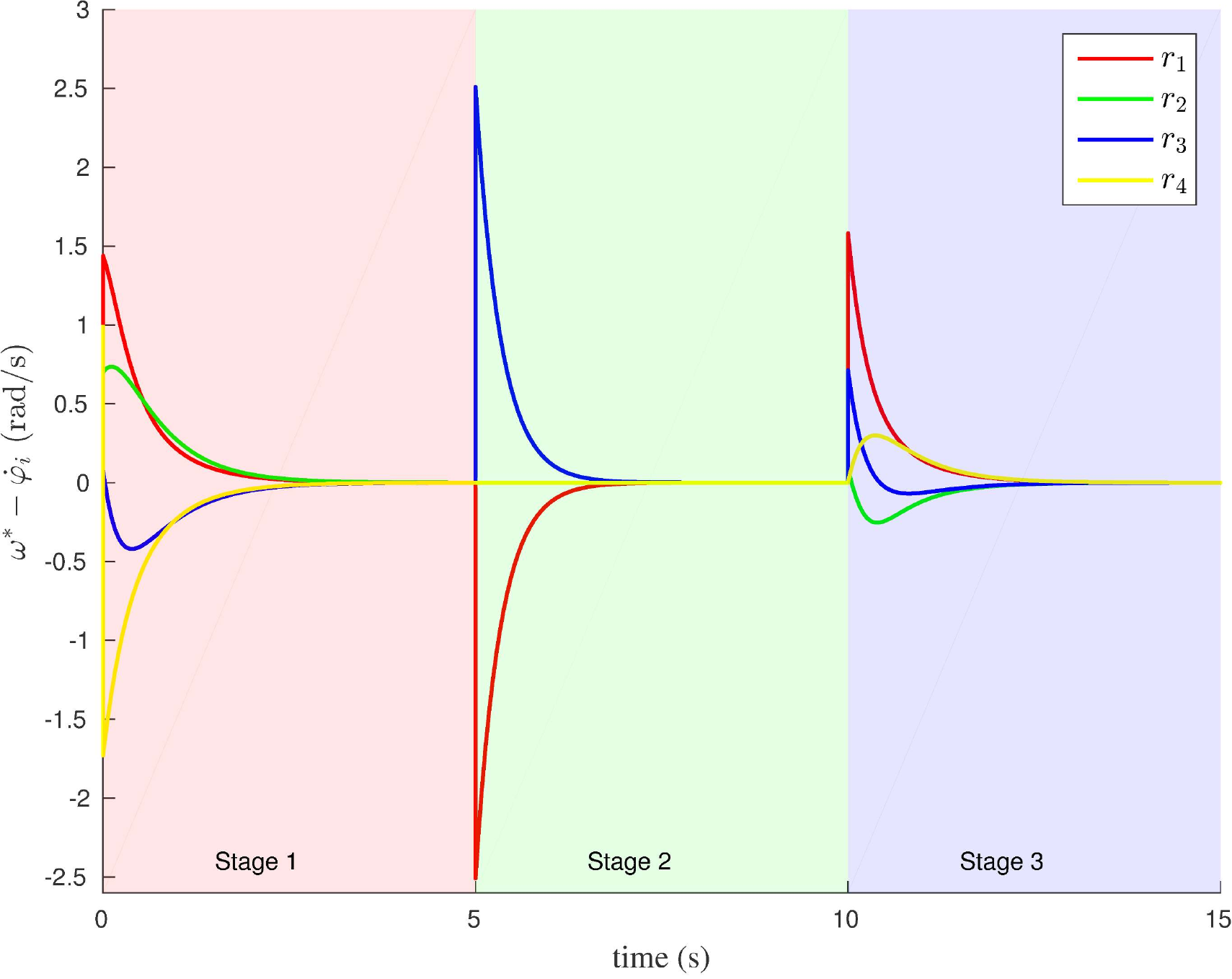}
			\caption{}
			\label{util_sim_w}
		\end{subfigure}
		\begin{subfigure}[b]{.4\columnwidth}
			\includegraphics[width=\linewidth]{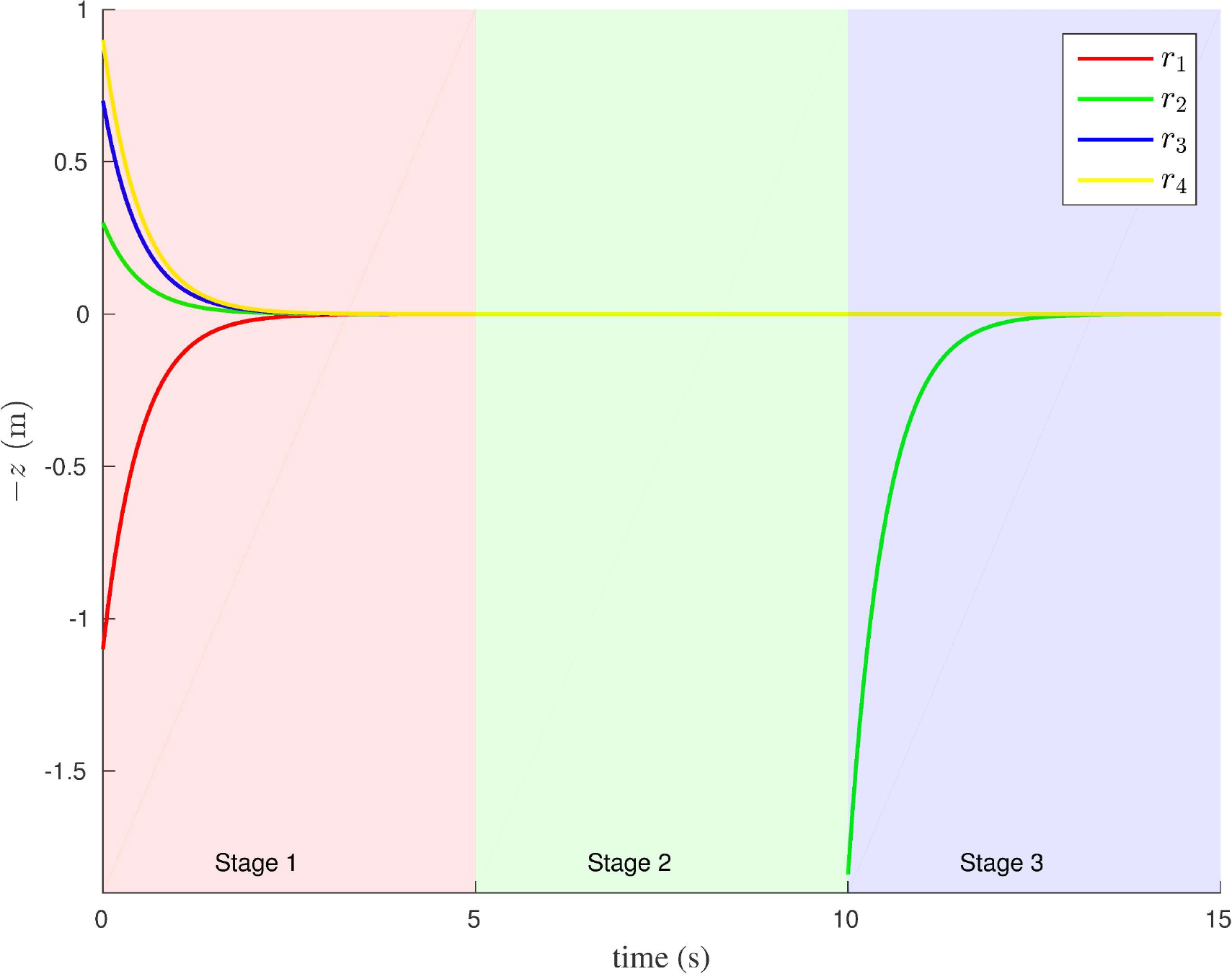}
			\caption{}
			\label{util_sim_z}
		\end{subfigure}
		\begin{subfigure}[b]{.4\columnwidth}
			\includegraphics[width=\linewidth]{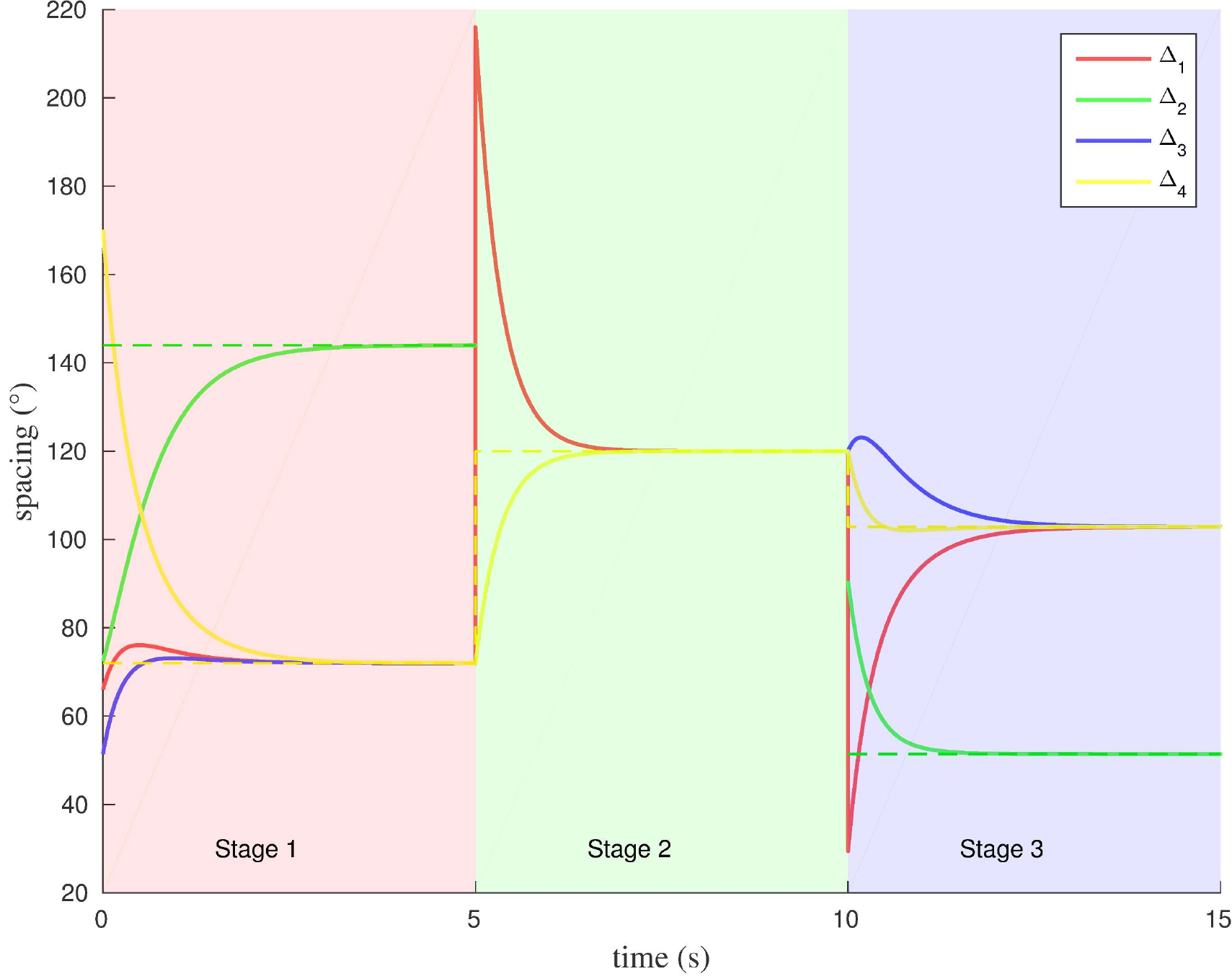}
			\caption{}
			\label{util_sim_spac}
		\end{subfigure}
		\caption{The data plots of the simulation. (a)-(c) are the error signals of the circumnavigation process. They are the plots for $\rho^*-\rho_i$, $\omega^*-\dot{\phi_i}$ and $-z$ respectively. (d) is the curves of the real spacing and the desired spacing. The red, green, blue and yellow dashed lines are the desired spacing for robots $r_1$, $r_2$, $r_3$ and $r_4$ respectively. }
		\label{utility_sim_data}
	\end{figure}
The circumnavigation process is illustrated in Fig. \ref{utility_sim}. The black dot in the center is the target to be encircled. Dashed lines connecting four robots represented by red, green, blue and yellow dots indicate the \textit{formation shape} and communication topologies intuitively. Dashed lines connecting the target and the robots demonstrate the circumnavigation radii. At Stage 1, four robots form a stable formation and circumnavigate the target (see Fig. \ref{util_sim0}). Then $r_2$ leaves the formation at Stage 2 (see Fig. \ref{util_sim1}). At the beginning of Stage 3, $r_2$ joins the formation (see Fig. \ref{util_sim2}). Note that the communication topologies have changed from Stage 1 to Stage 2 and from Stage 2 to Stage 3. Four robots form a stable formation again with different spacing in comparison with Stage 1 in the end (see Fig. \ref{util_sim3}).

The data plots of the simulation are shown in Fig. \ref{utility_sim_data}. Since $r_2$ quits from the formation during Stage 2, the corresponding data is not plotted. It can be seen from Fig. \ref{utility_sim_data} that although the changes of the utility of $r_2$ lead to the deviation of curves from the expected values, the circumnavigation error signals converge to zero exponentially during the three stages (see Fig. \ref{util_sim_rou}, Fig. \ref{util_sim_w} and Fig. \ref{util_sim_z} resp.). The spacing among robots changes according to the variation of the utilities. However, the spacing converges to the desired ones at the end of each stage (see Fig. \ref{util_sim_spac}).  Although the simulation experiment involves only four robots, it should be noted that the control algorithm only utilizes the information of neighboring robots, therefore it can be extended to a system with any number of robots. 
	\begin{table}[htb]
	 \renewcommand{\arraystretch}{1.3}
	 \centering
	 \begin{minipage}[t]{\columnwidth} 
	 \caption{The utilities and the corresponding spacing at four stages.}
	 \label{tab:utility}
	 \centering
	 	  \begin{tabular}{ccc}
	    \toprule[1pt]
	    \bfseries Stage  & Utilities  &  Desired spacing \footnote{~Under Formation Guideline \ref{fg1}, the desired spacing (unit: degree) are calculated according to the robots' utilities. However, this information is not needed by robots during the circumnavigation process.} \\
	    \midrule[0.5pt]
	    Stage 1 & $[20~~1~~20~~20]^T$  & $[62~~62~~118~~118]^T$ \\
	    Stage 2 & $[20~~20~~20~~20]^T$ & $[90~~90~~90~~90]^T$ \\
	    Stage 3 & $[20~~50~~20~~20]^T$ & $[114~~114~~66~~66]^T$ \\
	    Stage 4 & $[20~~0~~20~~20]^T$  & $[120~~120~~120]^T$ \footnote{~ At this stage, the robot $r_2$ quits from the formation.} \\
	    \bottomrule[1pt]
	  \end{tabular}
	 \end{minipage}
	\end{table}

\subsection{Experiment with Soccer Robots}

In this experiment, four soccer-playing robots \cite{Xiong2016,yao2015simulation} are used and Formation Guideline \ref{fg1} is adopted. Since the soccer-playing robots have omnidirectional movement abilities and they can reach any given velocity instantly, their dynamics can be regarded as the first-integrator model given in \eqref{eq1}. In addition, an omnidirectional vision system is equipped on each robot with algorithms for self-localization and the recognition of a yellow football \cite{lu2011robust}. The position and velocity of the robot itself and the position and velocity of the football are obtained by its own omnidirectional vision system. Moreover, robots are only allowed to receive information from its neighboring robots and the information is transmitted using wireless communication. 
%
%	\begin{equation} \label{eqmu2_2}
%	\mu_2 =
%	\begin{dcases}
%	1,  & 0 \le t <15,  \\
%	20, & 15 \le t <30, \\
%	50, & 30 \le t <45, \\
%	0,  & t \ge 45.
%	\end{dcases}
%	\end{equation}
%
	\begin{figure}[!htb]
		\centering %
		\begin{subfigure}[b]{.32\columnwidth}
			\includegraphics[width=\linewidth]{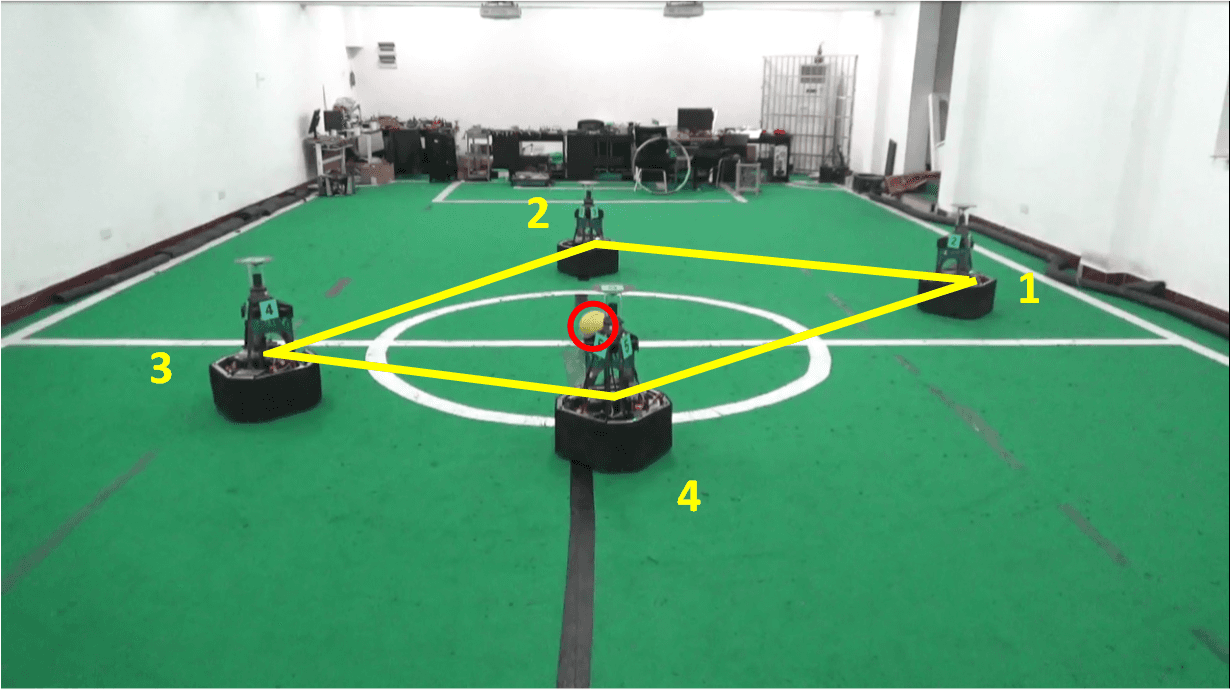}
			\caption{}
			\label{util_real0}
		\end{subfigure}
		\begin{subfigure}[b]{.32\columnwidth}
			\includegraphics[width=\linewidth]{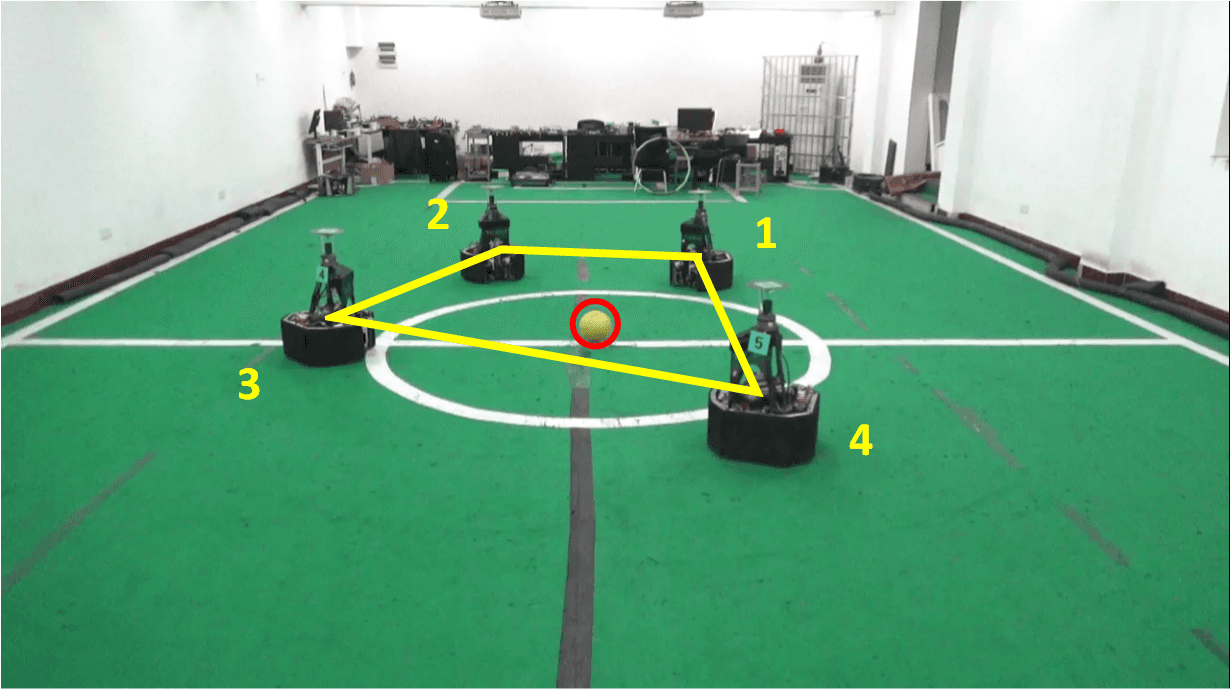}
			\caption{}
			\label{util_real1}
		\end{subfigure}
		\begin{subfigure}[b]{.32\columnwidth}
			\includegraphics[width=\linewidth]{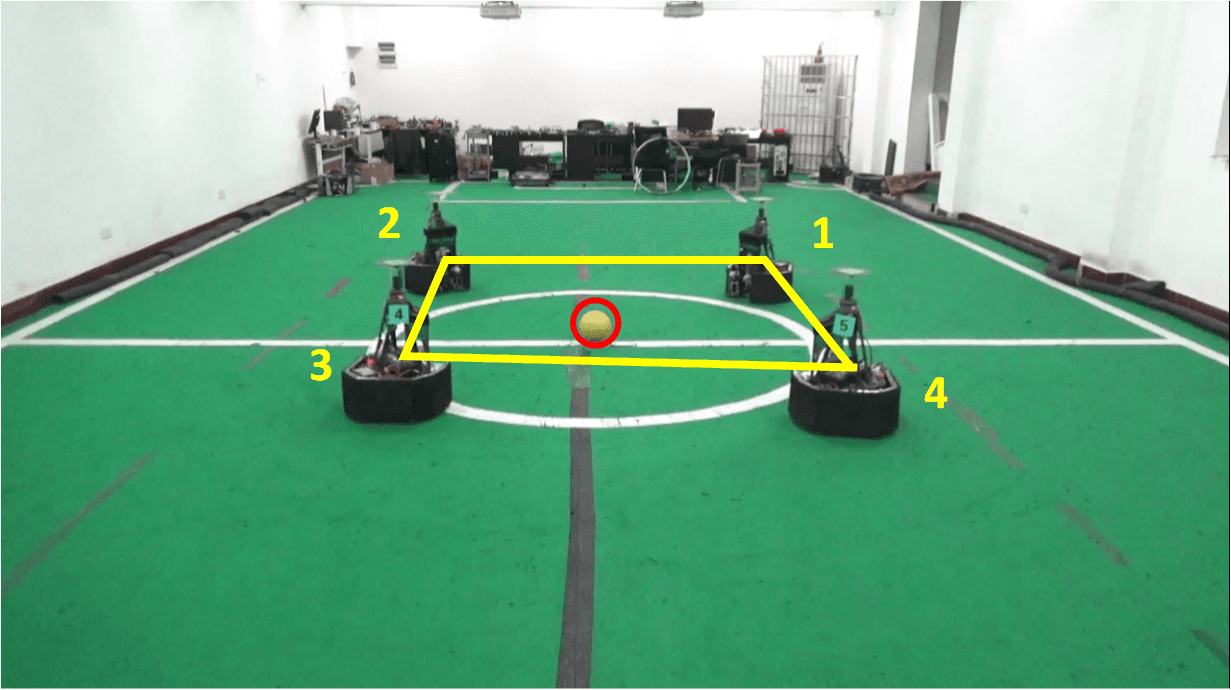}
			\caption{}
			\label{util_real2}
		\end{subfigure}
		\begin{subfigure}[b]{.32\columnwidth}
			\includegraphics[width=\linewidth]{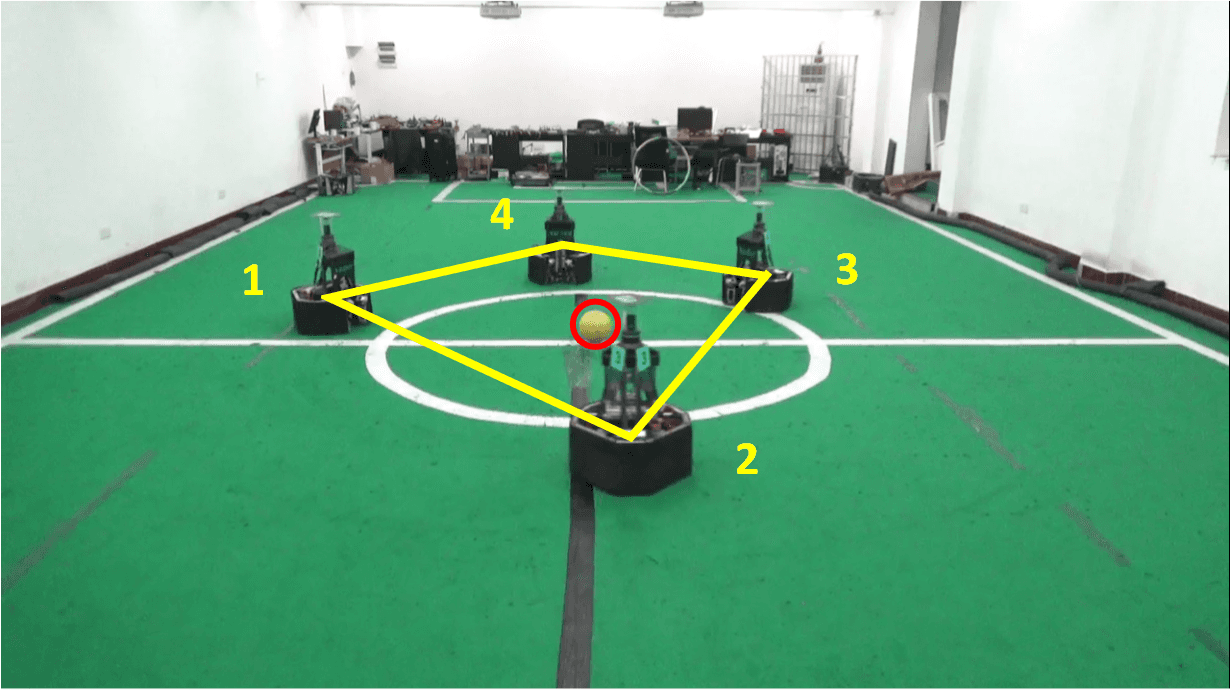}
			\caption{}
			\label{util_real3}
		\end{subfigure}
		\begin{subfigure}[b]{.32\columnwidth}
			\includegraphics[width=\linewidth]{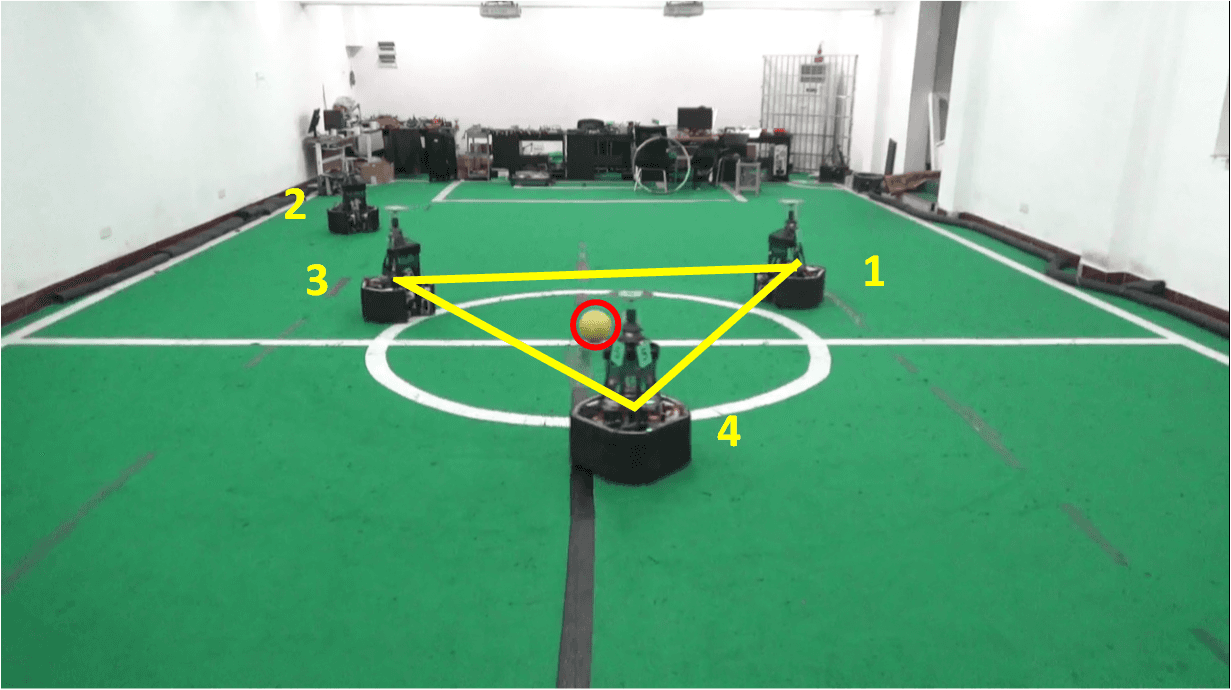}
			\caption{}
			\label{util_real4}
		\end{subfigure}
		\caption{The real robot experiment. (a) are the initial robot positions; (b)-(e) are positions at 4 s (Stage 1), 19 s (Stage 2), 39 s (Stage 3) and 61 s (Stage 4) respectively.}
		\label{utility_real}
	\end{figure}
	\begin{figure}[!htb]
		\centering %
		\begin{subfigure}[b]{.4\columnwidth}
			\centering
			\includegraphics[width=\linewidth]{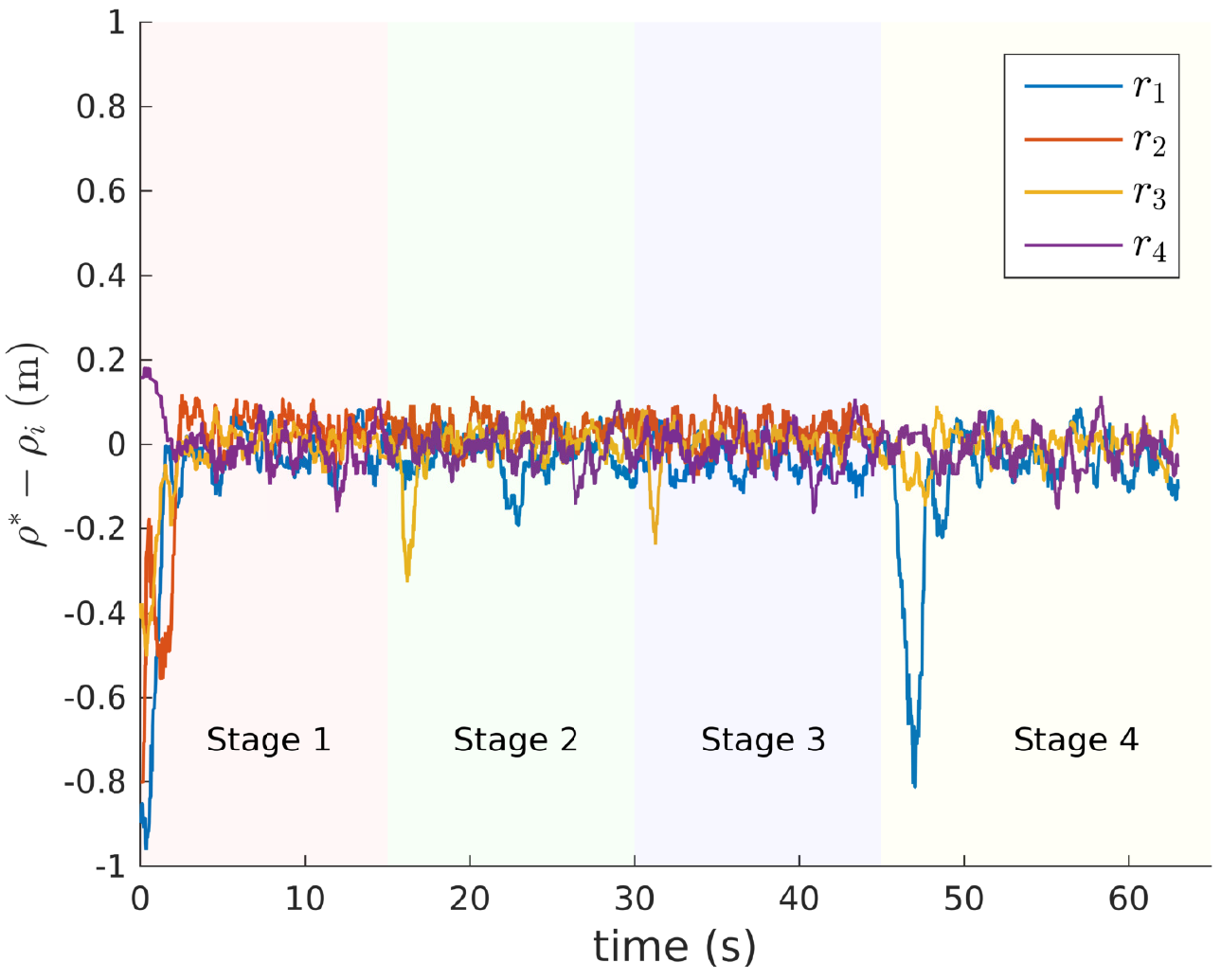}
			\caption{}
			\label{util_real_rou}
		\end{subfigure}
		\begin{subfigure}[b]{.4\columnwidth}
			\centering
			\includegraphics[width=\linewidth]{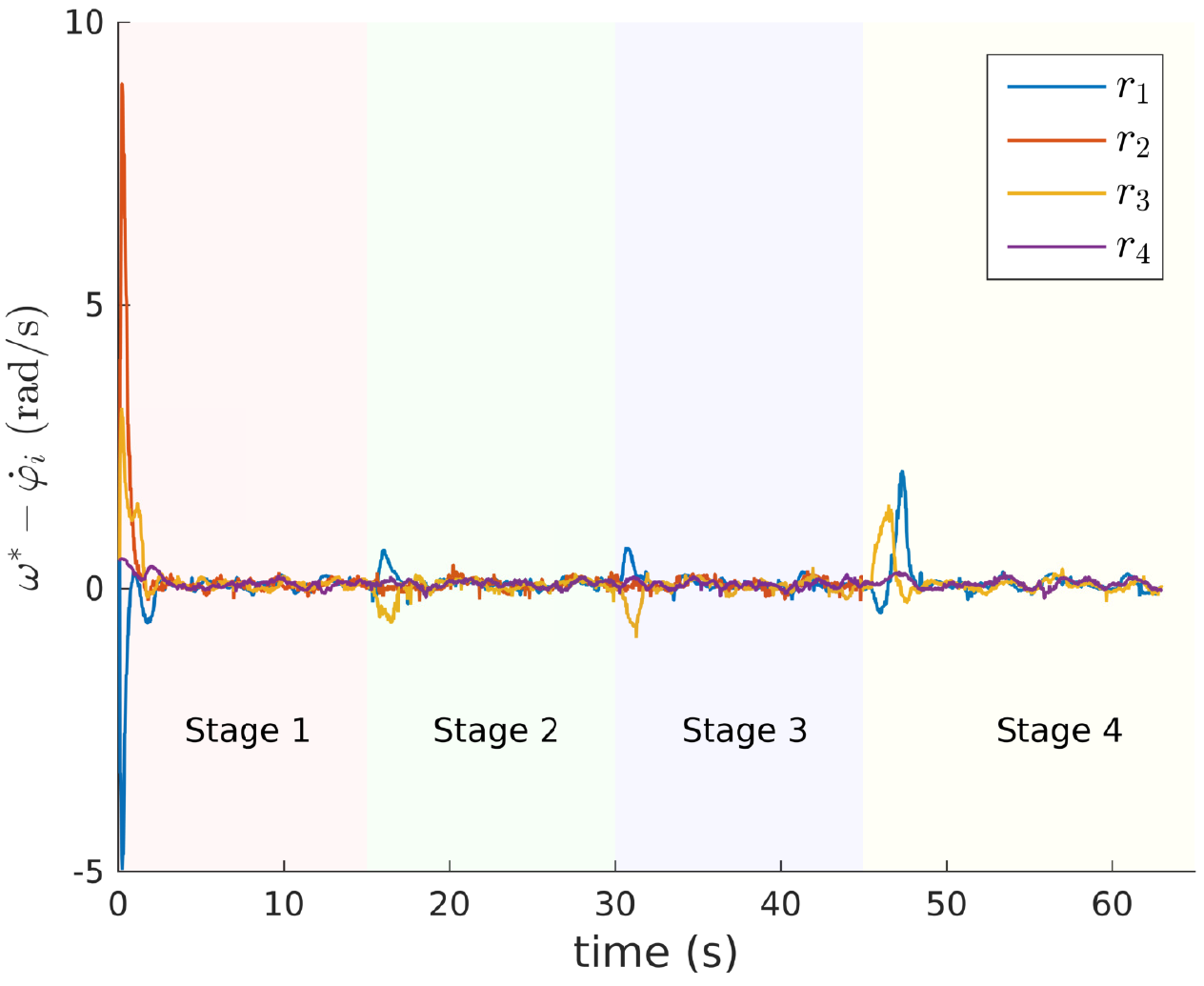}
			\caption{}
			\label{util_real_w}
		\end{subfigure}
		\begin{subfigure}[b]{.4\columnwidth}	
			\centering
			\includegraphics[width=\linewidth]{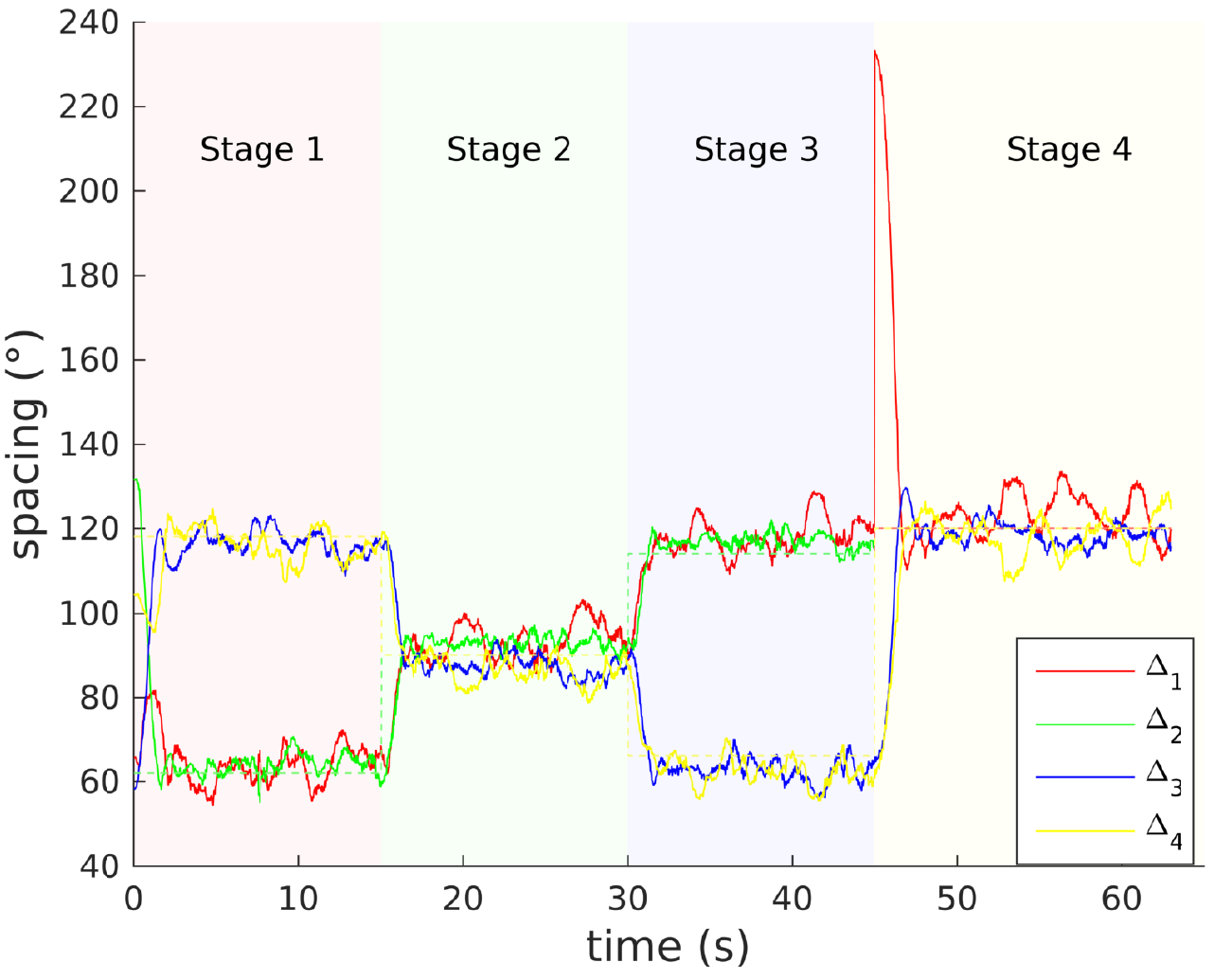}
			\caption{}
			\label{util_real_spac}
		\end{subfigure}
		\begin{subfigure}[b]{.4\columnwidth}
			\centering
			\includegraphics[width=\linewidth]{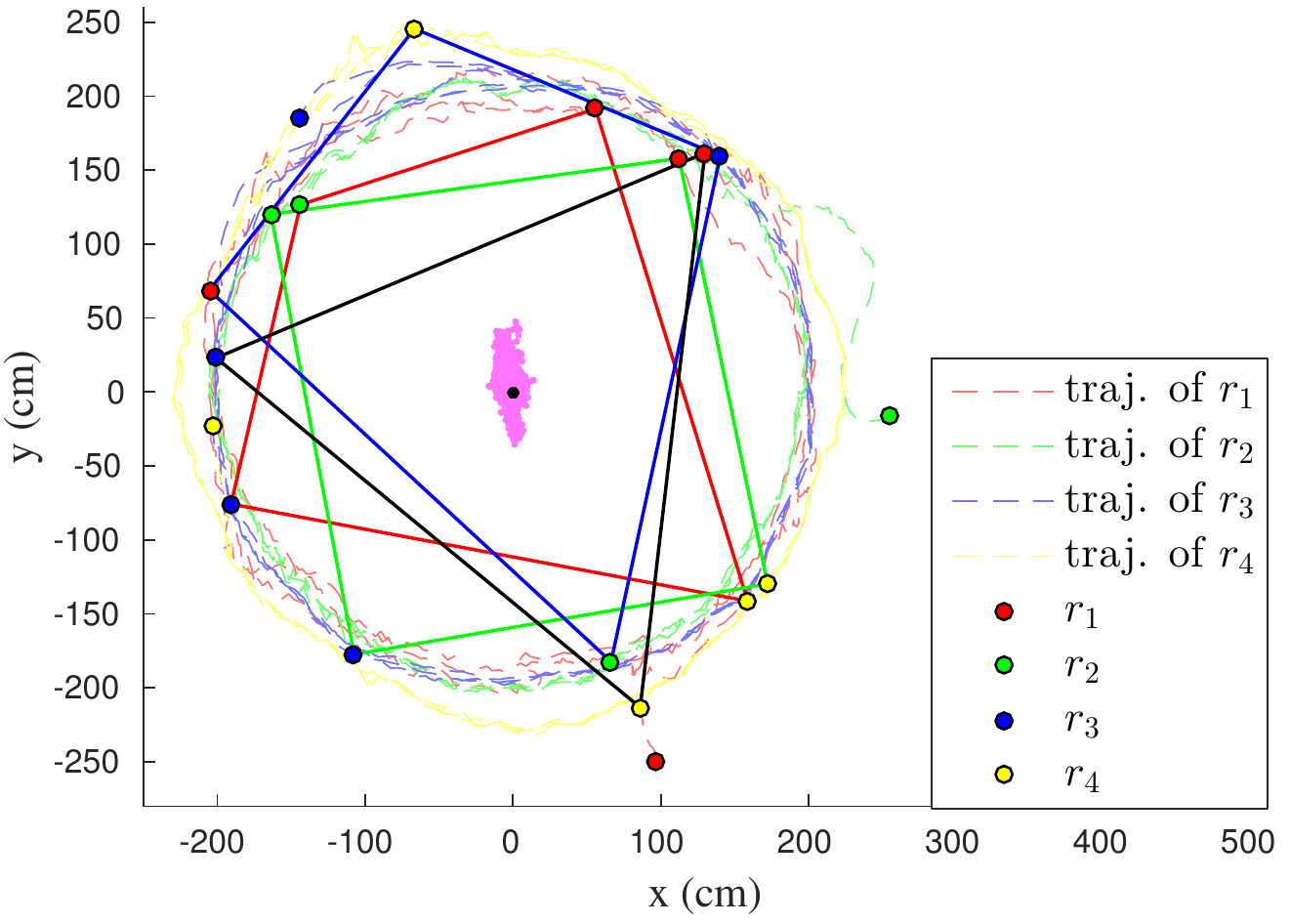}
			\caption{}
			\label{util_real_traj}
		\end{subfigure}
		\caption{The data plots of the real robot experiment. (a) and (b) are the error signals of the circumnavigation process. They are the plots for $\rho^*-\rho_i$ and $\omega^*-\dot{\phi_i}$ respectively. (c) is the curves of the real spacing and the desired spacing. The red, green, blue and yellow dashed lines are the desired spacing for robots $r_1$, $r_2$, $r_3$ and $r_4$ respectively. (d) is the plot of the trajectories of robots on the $X S Y$ plane.}
		\label{utility_real_data}
	\end{figure}

The utilities of robots $r_1$, $r_3$ and $r_4$ remain 20 throughout the whole circumnavigation process, while the utility of the robot $r_2$ varies according to a piecewise constant function. That is, $\mu_2=1, ~(0 \le t <15); ~\mu_2=20, ~(15 \le t <30); ~\mu_2=50, ~(30 \le t < 45); ~\mu_2=0, ~(t \ge 45)$. For convenience, the four time ranges are denoted by Stage 1, 2, 3 and 4 respectively. Note that at Stage 4, the robot $r_2$ quits from the circumnavigation process as its utility becomes zero. The utility of the four robots and the corresponding desired spacing is listed in Table \ref{tab:utility}. In this experiment, robots' initial positions are randomly chosen. The experiment parameters are $\rho^*=2$ m, $w^*=0.5$ rad/s, $k_\varphi=2.5$ and $k_\rho=2$. 

The circumnavigation process is shown in Fig. \ref{utility_real}. It demonstrates the positions of robots at different stages. The yellow lines connecting each robot's center indicate the formation shape. The ball in the middle of the field is the target to be encircled, which is marked by a red circle. The corresponding data plots are shown in Fig. \ref{utility_real_data}. Since robots move on the ground, the error plot of $-z$ is omitted. In Fig. \ref{util_real_traj}, the red, green, blue and black solid lines connecting the centres of robots represent the formation shapes at Stage 1, 2, 3 and 4 respectively. The dashed lines originated from robots are their trajectories. Note that since $r_2$ quits from the formation at Stage 4 ($\mu_2=0$), the data related to $r_2$ is not plotted after 45 s.

It can be seen from these figures that at Stage 1, since the utility of $r_2$ is the least, its neighboring robots $r_1$ and $r_3$ decrease their spacing with $r_2$ to compensate for this insufficiency (see Fig. \ref{util_real1} or the red solid lines in Fig. \ref{util_real_traj}). When it comes to Stage 2, robots form an equal-spacing formation as their utilities are equal (see Fig. \ref{util_real2} or the green solid lines in Fig. \ref{util_real_traj}). Stage 3 is contrary to Stage 1, where the utility of $r_2$ becomes the greatest. Therefore, its two neighboring robots increase the corresponding spacing with it (see Fig. \ref{util_real3} or the blue solid lines in Fig. \ref{util_real_traj}). Finally, at Stage 4, $r_2$ quits from the formation due to its zero utility. The other three robots form a new communication topology as stated in Remark \ref{remark5} (i.e., excluding the robot $r_2$) and transform to a equilateral triangular formation as their utilities are identical (see Fig. \ref{util_real4} or the black solid lines in Fig. \ref{util_real_traj}).

In addition, the circumnavigation radii, angular speeds and formation spacing converge to but fluctuate around the desired values at each stage (see Fig. \ref{util_real_rou}, Fig. \ref{util_real_w} and Fig. \ref{util_real_spac}). Moreover, at the three intersections of the stages (15 s, 30 s and 45 s respectively), the circumnavigation radii and the angular speeds of robots $r_1$ and $r_3$, the two neighboring robots of $r_2$, experience large deviation from the desired angular speed. They, however, converge swiftly afterwards (see Fig. \ref{util_real_w}). Noticeably,  at the last intersection (45 s), the circumnavigation radius and spacing for the robot $r_1$ deviate significantly from the desired values due to the absence of the robot $r_2$ in the formation, but the variations diminish rapidly subsequently (see Fig. \ref{util_real_rou} and Fig. \ref{util_real_spac}). The robot $r_4$ is hardly affected as it is not a neighboring robot of $r_2$. In Fig. \ref{util_real_traj}, the black dot at the center is the real position of the target and the cluster of pink dots are the perceived positions of the target by $r_1$. This manifests that information noise increases the uncertainty of the perceived information. Although there are fluctuations due to the information noise, the real spacing converges to the expected spacing at each stage (see Fig. \ref{util_real_spac}). Fig. \ref{util_real_traj} shows that the formation spacing changes dynamically as the utility of $r_2$ varies. 

\section{Concluding Remarks and Future Work} \label{sec5}

This paper proposes a distributed control law for a multi-robot system to realize circumnavigation process with dynamic spacing based on utilities. Unlike most of the existing study, in this paper, the spacing is not fixed and equal but they are dynamic, which is useful if robots are heterogeneous (e.g. with different kinematics capabilities). There is no assumption that the robots should be initially placed on a prescribed circle nor should they splay evenly on the circle. The control law is distributed and scalable. The theoretical analysis using graph theory along with the simulation and real-robot experiments prove the effectiveness of the proposed circumnavigation control algorithm based on utilities. 

Although Theorem \ref{theorem2} implies that robots will not collide with each other since their orders are unchanged during the circumnavigation process, this claim is based on the assumption that robots are considered as mass points. The collision avoidance problem taking into account the physical dimensions of robots will be studied in the future. 

\begin{acknowledgements}
	Our work is supported by National Science Foundation of China (NO. 61503401 and NO. 61773393), and graduate school of National University of Defense Technology.
\end{acknowledgements}

%\appendix

%\begin{remark}
%The proof of this theorem adds an additional positive gain $k_\varphi$ compared with that of Theorem 1 in \cite{wang2013} about order preservation. In addition, this proof reveals that \eqref{eqdelta} represents an Metzler dynamics system, which has already been studied extensively (e.g., \cite{mitkowski2008dynamical}). 
%\end{remark}

\bibliographystyle{splncs03}
\bibliography{ref/ref}

\end{document}